\documentclass[twoside]{article}

\usepackage[accepted]{config/aistats2022}

\usepackage[round]{natbib}

\usepackage{pdflscape}

\newcommand{\bfI}{{\bf I}}

\newcommand{\bfW}{{\bf W}}

\newcommand{\bfY}{{\bf Y}}
\newcommand{\bfZ}{{\bf Z}}

\newcommand{\bfb}{{\bf b}}

\newcommand{\bfh}{{\bf h}}

\newcommand{\bfx}{{\bf x}}
\newcommand{\bfy}{ {\bf y}}
\newcommand{\bfu}{{\bf u}}

\newcommand{\bfv}{{\bf v}}
\newcommand{\bfw}{{\bf w}}

\newcommand{\bfz}{{\bf z}}

\newcommand{\bfalpha}{{\boldsymbol \alpha}}

\newcommand{\bfphi}{{\boldsymbol \Phi}}

\newcommand{\bftheta}{{\boldsymbol \theta}}

\usepackage{amsmath,amsfonts,amssymb}
\usepackage{xcolor}
\usepackage{tikz}
\usetikzlibrary{arrows,calc}
\usepackage{relsize}
\usepackage{pgfplots}
\usepackage{tikz}
\usetikzlibrary{plotmarks}
\usepackage{standalone}
\usepackage{graphicx}
\graphicspath{{figures/}}

\usepackage{bm}

\usepackage{microtype}
\usepackage{bm}
\usepackage{graphicx}
\usepackage{subcaption}
\usepackage{booktabs} 
\usepackage{siunitx} 
\usepackage{etoolbox}
\newcommand{\ubold}{\fontseries{b}\selectfont}
\robustify\ubold

\usepackage{hyperref}

\usepackage{url}       
\usepackage{amsfonts, amssymb, amsmath}       
\usepackage{multirow}
\usepackage{amsthm}
\usepackage{tikz}

\newtheorem{proposition}{Proposition}

\makeatletter
\DeclareRobustCommand{\cev}[1]{%
  {\mathpalette\do@cev{#1}}%
}
\newcommand{\do@cev}[2]{%
  \vbox{\offinterlineskip
    \sbox\z@{$\m@th#1 x$}%
    \ialign{##\cr
      \hidewidth\reflectbox{$\m@th#1\vec{}\mkern4mu$}\hidewidth\cr
      \noalign{\kern-\ht\z@}
      $\m@th#1#2$\cr
    }%
  }%
}
\makeatother
\newcommand{\flow}{{\cev{g}}}
\newcommand{\Flow}{{\cev{G}}}
\newcommand{\MQF}{MQF$^2$}

\newcommand{\Expect}{\mathbb{E}}

\newcommand\blfootnote[1]{%
  \begingroup
  \renewcommand\thefootnote{}\footnote{#1}%
  \addtocounter{footnote}{-1}%
  \endgroup
}

\begin{document}

\runningauthor{Kan, Aubet, Januschowski, Park, Benidis, Ruthotto, Gasthaus}

\twocolumn[

\aistatstitle{Multivariate Quantile Function Forecaster}

\aistatsauthor{Kelvin Kan$^{1,\dagger}$
\And Fran\c{c}ois-Xavier Aubet$^2$
\And Tim Januschowski$^{3,\dagger}$
}
\aistatsauthor{Youngsuk Park$^2$
\And Konstantinos Benidis$^2$ 
\And Lars Ruthotto$^1$
}
\aistatsauthor{Jan Gasthaus$^2$}
\aistatsaddress{$^1$Emory University 
\And $^2$Amazon Research
\And $^3$Zalando SE}
]

\begin{abstract}
We propose Multivariate Quantile Function Forecaster (\MQF), a global probabilistic forecasting method constructed using a multivariate quantile function and investigate its application to multi-horizon forecasting.
Prior approaches are either autoregressive, implicitly capturing the dependency structure across time but exhibiting error accumulation with increasing forecast horizons, or multi-horizon sequence-to-sequence models, which do not exhibit error accumulation, but also do typically not model the dependency structure across time steps. \MQF\ combines the benefits of both approaches, by directly making predictions in the form of a multivariate quantile function, defined as the gradient of a convex function which we parametrize using input-convex neural networks. By design, the quantile function is monotone with respect to the input quantile levels and hence avoids quantile crossing. We provide two options to train \MQF: with energy score or with maximum likelihood. Experimental results on real-world and synthetic datasets show that our model has comparable performance with state-of-the-art methods in terms of single time step metrics while capturing the time dependency structure.
\end{abstract}

\section{INTRODUCTION}
\label{sec:intro}
Among the many applications of time series forecasting (see e.g.,~\citet{petropoulos2021forecasting} for an overview), inventory management in supply chain contexts has a prominent place. For this use-case in particular, \emph{probabilistic} forecasts provide the input to downstream decision making problems such as replenishment decisions. For example, variations of the classic newsvendor problem show a direct correspondence between different quantile levels of a probabilistic forecast distribution with safety stocks in inventory management. It is therefore no surprise that recently proposed probabilistic forecasting methods have considered quantiles or the quantile function to represent probabilistic forecasts in the univariate case~\citep{gasthaus2019,wen2017,park2021,gouttes2021}. \blfootnote{\hspace{-.2cm}$^\dagger$Work done while at Amazon Research.}

In this present work, we extend and generalize existing work by considering multivariate quantile functions. Existing quantile-based methods make predictions in the form of univariate quantiles (or quantile functions), i.e.,\ independently for multiple time points in a multi-horizon setting, or independently across items in a multivariate forecasting setting (or both), and thereby ignore existing dependency structures in their forecasts.  The extension to multivariate quantile functions allows us to capture these dependencies, which can have a significant impact on the accuracy of downstream systems (e.g.,\ automatic inventory management) by capturing effects such as cannibalization, cross-selling or substitutability of products~\citep{zhang2014,RAJARAM2001582,hanasusantoKWZ15}.
Despite their benefits, we note that multivariate quantile functions have not been studied or applied extensively in the literature on forecasting or machine learning (although there is a growing body of work from statistics and econometrics). We speculate that this may be due to the fact that the generalization of quantile functions from univariate to multivariate is not unique and exploring the properties of different notions of multivariate quantiles is still an active area of research. In contrast, multivariate probabilistic forecasts represented as multivariate probability densities have recently received  more attention~\citep{rasul2020,salinas2019,de2020normalizing,rasul2021autoregressive}. 

In the univariate case, the quantile function of a random variable is (loosely speaking) the inverse of the cumulative distribution function (CDF). In the multivariate case, however, the corresponding notion of a \emph{multivariate quantile function} is not uniquely defined, and in fact several proposals have been presented (see \citet[Sec 2.4]{carlier2016}), emphasizing different attributes of univariate quantile functions. Here we adopt the definition of \citet{carlier2016} (studied further in \citet{chernozhukov2017, hallin2021}) as maps that (i) map a reference distribution (e.g.,\ uniform on the unit cube) to the target distribution, and (ii) are monotonic, where the particular notion of multivariate monotonicity used is that of being the gradient of a convex function. Through Brenier's theorem~\citep{brenier1991} and Knott-Smith optimality criterion~\citep{knott1984}, this definition characterizes multivariate quantile functions as the unique solutions to optimal transport problems with quadratic costs. Further, property (i) immediately connects this notion to normalizing flows (where typically the inverse direction is parametrized), and indeed normalizing flows inspired by this particular form of monotonicity has recently been proposed \citep{onken2021ot,huang2020}.

In more details, our contributions are as follows:
\begin{itemize}
    \item Building on the notion of multivariate quantile functions as gradients of convex function put forth in \citet{carlier2016}, we propose to parametrize multivariate quantile functions via the gradients of \emph{input convex neural networks}~\citep{amos2017}. 
    \item We propose a training procedure based on the \emph{energy score}~\citep{gneiting2007}, a generalization of the continuous ranked probability score~\citep{matheson76} to the multivariate case, and empirically demonstrate that this is effective and robust. Our model can alternatively be trained using a more standard maximum likelihood estimation approach, by relating it to normalizing flows (in particular the convex potential flows proposed in~\citet{huang2020}).
    \item We combine the multivariate quantile function model with an RNN-based feature extractor, resulting in a forecasting method that yields accurate joint multi-step forecasts. To the best of our knowledge, we are the first to represent multivariate forecasts using multivariate quantile functions.
\end{itemize}

In our empirical evaluations we show the practical viability of our approach in a series of experiments on both real-world and synthetic data where we employ the multivariate quantile function to model the multi-step forecast distribution.  Our approach avoids pitfalls like error accumulation~\citep{salinas2020} and quantile crossing~\citep{wen2017} while allowing for realistic samples from the probabilistic forecast. The latter is particularly important in applications that require human interaction, e.g., in a supply chain context, where business analysts want to consider extreme scenarios to sharpen their intuition about the future.

The rest of the paper is organized as follows. In Section \ref{sec:background} we review the building blocks of our methodology: multivariate quantile functions and various training objectives. In Section \ref{sec:MQF2} we present our model and describe the training and inference procedures. In Section \ref{sec:experiments} we provide an empirical evaluation on several real-world datasets and conclude the paper in Section \ref{sec:conclusion}.  We start by reviewing the state of the art. 

\section{RELATED WORK}
\label{sec:related}
Deep learning-based approaches to probabilistic time series forecasting have been widely studied (see \citet{benidis2020} and references therein). In addition to models utilizing parametric distributions (e.g.,\ DeepAR \citep{salinas2019}), approaches based on quantile regression \citep{koenker1978, koenker2005} combined with RNN/CNN \citep{wen2017} or Transformer-based \citep{li2019, lim2021} feature extractors have been shown to be flexible and effective. However, these approaches are limited to univariate predictions at pre-specified quantile levels and suffer from the quantile crossing problem. Recent work on modeling univariate quantile functions \citep{gasthaus2019, park2021} has addressed these limitations while still focusing on the univariate case. Our approach extends this work to multivariate quantile functions, and similarly does not suffer from quantile crossing or require quantile level pre-specification.

The idea of using (univariate) quantile levels as input to a neural network in order to define a flexible quantile function model has previously been explored. \citet{dabney2018} proposed implicit quantile networks which are trained by minimizing quantile loss using random uniform samples as input in the context of distributional reinforcement learning to model the state-action return distribution, and \citet{gouttes2021} employed the same approach in the context of time series forecasting. A similar approach using uniform samples as input and minimizing the corresponding quantile loss has been proposed in \citet{tagasovska2019} as a generic mechanism for modeling aleatoric uncertainty. However, none of these approaches explicitly enforce the monotonicity constraint on the quantile function, nor consider multivariate quantiles.

Other multivariate notions of quantile functions than the one we use here have been proposed, e.g.,\ through univariate conditional quantile functions (requiring the choice of an ordering) \citep{wei2008}, or as gradients of convex potentials but without requiring them to transport from a reference distribution to the target distribution \citep{koltchinskii1997}. The notion we make use of here allows us to easily obtain samples (by being an optimal transport) while not requiring the choice of a particular ordering of the dimensions.

Conceptually closest to our approach, although not proposed in the setting of time series forecasting, is the work on convex potential flows \citep{huang2020}. As in our work, and also inspired by the connection to optimal transport through Brenier's theorem, the authors propose to define an invertible model as the gradient of a convex function and demonstrate how the inverse as well as the Jacobian determinant required for likelihood-based learning can be computed efficiently. They also propose to use input-convex neural networks to model the underlying convex function. Their work does not, however, consider the invertible mapping as a multivariate quantile function and---like other work on normalizing flows trained using maximum likelihood---uses a parametrization in the ``reverse'' direction (from the target to the Gaussian reference distribution). The same idea of using the gradient of an input-convex neural network model to define functions that are solutions to optimal transport problems under quadratic cost has also been proposed in \citet{bunne2021}, albeit embedded in a larger architecture for modeling population dynamics.

More broadly, density models defined through invertible maps (``normalizing flows'') \citep{kobyzev2021} have been used in the context of time series forecasting as flexible uni- and multivariate density models: \citet{rasul2020} proposed to directly parametrize a multivariate forecast distribution using a normalizing flow, while \citet{de2020normalizing} combined a linear-Gaussian dynamical system with an invertible output model. These approaches, however, treat the flows as generic density estimators and do not draw the connection to multivariate quantile functions.

\section{BACKGROUND}
\label{sec:background}
In this section we introduce the necessary background and building blocks of our approach: multivariate quantile functions, the energy score (which forms the basis of our training procedure), normalizing flows (which underlie the alternative maximum likelihood training procedure), and (partially) input convex neural networks (which we use to parametrize the convex function).

\subsection{Multivariate Quantile Functions}
\label{subsection:Multivariate Quantile Functions}
In the univariate case, for a real-valued random variable $Z$, denote by $F_Z(z)$ its CDF. The corresponding quantile function is defined as
\begin{equation}
    q_Z(\alpha) := F_Z^{-1}(\alpha) = \inf \{ z\in \mathbb{R} : \alpha \leq F_Z(z) \}.
\label{eq:univariate_quantile}
\end{equation}
Here $\alpha \in (0,1)$ is the quantile level, which is the probability that $Z$ is less than $q_Z(\alpha)$. 

While the CDF naturally extends to the multivariate case ($F_{\bfZ}: \mathbb{R}^d \rightarrow (0, 1)$), it is not invertible in general, precluding us from defining the corresponding quantile functions as its inverse. One natural way to define a quantile function of an $n$-variate random variable $\bfZ$ is as a mapping from $(0,1)^n$ to $\mathbb{R}^n$. The input to this mapping is a quantile vector $\bfalpha \in (0,1)^n$ (instead of a single quantile level $\alpha$), where the $i$-th entry $\alpha_i$ represents the quantile level of $[q_Z(\bfalpha)]_i$. However, such a mapping is not uniquely defined, as the entries of the quantile vector can interact with each other, so that the quantile levels do not have the same probabilistic meaning~\citep{carlier2016}. The definition proposed in \citet{carlier2016} that we adopt here resolves this ambiguity by enforcing a particular notion of monotonicity.

In the univariate case, by construction, the quantile function has two essential properties. The first property is satisfying the representation property\footnote{In general, quantile vectors can follow distributions other than $U(0,1)^n$~\citep{carlier2016}. For instance, one can use the isotropic Gaussian distribution.}
\begin{equation}
   \bfZ = q_{\bfZ}(\bfalpha), \quad \bfalpha \sim \mathit{U}(0,1)^n,
\label{eq:multivariate_quantile}
\end{equation}
for $n=1$. Here, we slightly abused notation by denoting $\bfalpha$ as a random variable. The second property is monotonicity, i.e.,
\begin{equation}
\text{if } \alpha_1 < \alpha_2\text{, then }q_Z(\alpha_1) \leq q_Z(\alpha_2). 
\label{eq:property_monotonicity}
\end{equation}
The multivariate (vector) quantile function proposed in \citet{carlier2016} is defined as a gradient of a convex function (a multivariate notion of monotonicity, thus extending \eqref{eq:property_monotonicity}) that satisfies the representation property \eqref{eq:multivariate_quantile}. Moreover, the convexity implies 
\begin{equation}
    (q_{\bfZ}(\bfalpha_1)-q_{\bfZ}(\bfalpha_2))^\top (\bfalpha_1-\bfalpha_2) \geq 0,
\label{eq:multivaraite_monotonicity}
\end{equation}
which reduces to the monotonicity in the univariate case. Thus, this definition reduces to the classical quantile function in the univariate case. By defining the quantile function to be monotonic in the sense of being the gradient of a convex function, a connection to work on optimal transport \citep{villani2009optimal, peyre2019} is drawn, where Brenier's polar factorization theorem~\citep{brenier1991} and Knott-Smith optimality criterion~\citep{knott1984} establish that such functions are the unique optimal transports with quadratic cost. In particular, the representation property and monotonicity are necessary and sufficient for multivariate quantile functions to be the unique optimal transports from a reference distribution (typically chosen to be uniform on the unit cube) to the distribution of interest under quadratic cost, i.e.\ they minimize the Wasserstein distance $\min\limits_{q: q(\bfalpha)=\bfZ} \Expect_{\bfalpha \sim U(0,1)^n}\| \bfalpha-q(\bfalpha)\|^2_2$.

\subsection{Energy Score}

In the univariate case, given realizations $z$ of the random variable $Z$, we seek to estimate the quantile function $q_Z$ for all quantile levels $\alpha \in (0,1)$. We can achieve this by minimizing the continuous ranked probability score (CRPS) \citep{gneiting2007}, defined as
\begin{equation}
    L_{\text{CRPS}}(q,z) = \Expect_{w, w' \sim q(\mathit{U}(0,1))} \left[ -\frac{1}{2} |w-w'| + |w-z| \right],
\end{equation}
where $w$ and $w'$ are independent. CRPS is strictly proper~\citep{gneiting2007}, i.e.,
\begin{equation}
    \Expect_{z \sim Z} L_{\text{CRPS}}(q_Z,z) < \Expect_{z \sim Z} L_{\text{CRPS}}(q,z),
\end{equation}
for any $Z$ and $q \neq q_Z$, both with finite first moment. In other words, for realizations $z$ of the random variable $Z$, the unique minimizer of CRPS is the quantile function $q_Z$.

The energy score~\citep{gneiting2007} is an extension of the CRPS to the multivariate setting, which takes a statistical energy perspective from \citet{szekely2003}. It is defined as
\begin{align}
\begin{split}
    L_{\text{ES}}(q,\bfz) &= \Expect_{\bfw, \bfw' \sim q(\mathit{U}(0,1)^n)} \bigg[ -\frac{1}{2}  \|\bfw-\bfw'\|^\beta_2 \\
    & \qquad\qquad\qquad\qquad\quad + \|\bfw-\bfz \|^\beta_2 \bigg],
\end{split}
\label{eq:energy_score}
\end{align}
where $\bfw$ and $\bfw'$ are independent.
If $\beta \in (0,2)$, the energy score is strictly proper~\citep{szekely2003} for any $\bfZ$ and $q \neq q_{\bfZ}$ satisfying $\Expect_{\bfz \sim \bfZ} \| \bfz \|^\beta_2 < \infty$ and $\Expect_{\bfw \sim q(\mathit{U}(0,1)^n)} \| \bfw \|^\beta_2 < \infty$, respectively.

In practice, we approximate the energy score by
\begin{align}
\begin{split}
\tilde{L}_{\text{ES}}(q,\bfz) =  & -\frac{1}{2 |\mathcal{C}||\mathcal{C}'|} \sum_{\substack{\bfw \in \mathcal{C},  \bfw' \in \mathcal{C}'}} \|\bfw-\bfw'\|^\beta_2 \\
& + \frac{1}{|\mathcal{C}''|} \sum_{\bfw'' \in \mathcal{C}''}\|\bfw''-\bfz\|^\beta_2
\end{split}
\label{eq:approx_energy_score}
\end{align}
where $\mathcal{C}$s are sets of finite samples drawn from $q(\mathit{U}(0,1)^n)$. 

\subsection{Normalizing Flows}
Normalizing flows \citep{tabak2013,ruthotto2021} are $C^1$-diffeomorphic and orientation-preserving functions which map from $\mathbb{R}^n$ to $\mathbb{R}^n$. In particular, they transform a random variable of interest $\bfZ$ with density $p_\bfZ$ into $\bfY$ with a simple density $p_\bfY$ which can be easily evaluated, typically an isotropic Gaussian. Note that the mappings of normalizing flows go in the opposite direction to quantile functions, which map samples of a uniform (simple) distribution to $\bfZ$. To distinguish between the two opposite directions, in this paper we denote normalizing flows as $\flow$.

\paragraph{Maximum Likelihood Estimation} Using the change of variables formula, we estimate $p_\bfZ(\bfz)$ as 
\begin{equation}
    p_\bfZ(\bfz) \approx p_\flow (\bfz) = p_\bfY(\flow(\bfz)) \det \left( \frac{\partial \flow(\bfz)}{\partial \bfz} \right),
    \label{eq:change_of_variable}
\end{equation}
where $p_\flow$ is the density induced by $\flow$. We target to find a normalizing flow that approximates the true density $p_\bfZ$ well. 
One method to evaluate the discrepancy between the two densities in \eqref{eq:change_of_variable} is the Kullback-Leibler (KL) divergence defined by \vspace{-10pt}
\begin{align*}
    \text{KL}(p_\bfZ \| p_\flow)& \!= \Expect_{\bfz \sim \bfZ} \left[\log \left(\frac{p_\bfZ(\bfz)}{p_\flow(\bfz)}\right) \right] \\
    & \! = \Expect_{\bfz \sim \bfZ} \log (p_\bfZ(\bfz)) \!- \!\Expect_{\bfz \sim \bfZ} \log (p_\flow(\bfz)).
\end{align*}
The first term is a constant and can be dropped in minimization. Replacing the expectation of the second term by samples $\bfz_i \sim \bfZ$, for $i=1,...,m$, we obtain the negative log-likelihood given by
\begin{align}
    \begin{split}
   & \frac{1}{m}\sum_{i=1}^m \tilde{L}_{\text{ML}}(\flow(\bfz_i)) := \frac{1}{m}\sum_{i=1}^m [-\log(p_\flow(\bfz_i))] \\ 
    & = \frac{1}{m} \sum_{i=1}^m \left[ - p_\bfY(\flow(\bfz_i)) - \log \left[ \det \left( \frac{\partial \flow(\bfz_i)}{\partial \bfz} \right) \right] \right]. 
    \end{split}
    \label{eq:likelihood}
\end{align}
\paragraph{Sample Generation} After training the normalizing flow, we can generate predicted samples of $\bfZ$ by going backward through the flow. That is computing the inverse $\flow^{-1}(\bfy)$, where $\bfy$ is a sample drawn from the reference distribution defined by $p_\bfY$.

\subsection{Partially Input Convex Neural Network}
Input convex neural network (ICNN) \citep{amos2017} is a neural network with special constraints on its architecture such that it is convex with respect to (a part of) its input. ICNN has demonstrated successful applications in various optimal transport and optimal control problems~\citep{bunne2021,chen2018,huang2020,makkuva2020}. Moreover, it has been proved that, under mild assumptions, ICNN and its gradient can universally approximate convex functions \citep{chen2018} and their gradients \citep{huang2020}, respectively. This means that the gradient of ICNN can universally approximate multivariate quantile functions in the marginal case. Moreover, we will demonstrate in our experiments that it can also effectively approximate conditional quantile functions.

We consider in this work a type of the ICNN called the partially input convex neural network (PICNN) \citep{amos2017}, we follow the PICNN architecture used in~\citep{huang2020}. For an input pair $(\bfalpha,\bfh) \in (\mathbb{R}^n \times \mathbb{R}^d)$, the key feature of PICNN is that it is only convex with respect to $\bfalpha$. The $i$-th layer of a $k$-layer PICNN, with $i=1, \dots, k$, is represented as:
\begin{align}
\bfv_{i+1} &= a_i \Big( \bfW^{(\bfv)}_i \left( \bfv_i \circ [\bfW^{(\bfv \bfu)}_i \bfu + \bfb^{(\bfv)}_i]_{+} \right) \\
&  + \bfW^{(\bfalpha)}_i \left( \bfalpha \circ (\bfW^{(\bfalpha \bfu)}_i \bfu + \bfb^{(\bfalpha)}_i)\right) + \bfW^{(\bfu)}_i \bfu + \bfb_i \Big), \nonumber
\\ &\text{with:}~~G_\bftheta(\bfalpha,\bfh) = \bfv_k, \label{eq:final_PICNN} 
\bfu = a (\tilde{\bfW} \bfh + \tilde{\bfb}).
\end{align}
Here, $\bfW$'s and $\bfb$'s are weights and bias of the network, respectively, they are collectively denoted as $\bftheta$, $\circ$ denotes the Hadamard product, and $[\cdot]_{+}$ denotes the ReLU activation function. Moreover, to render the convexity of the network, $a_i$'s are convex and non-decreasing activation functions, and $\bfW^{(\bfv)}_i$'s have non-negative entries.

\begin{figure*}[ht]
    \centering 
    \includegraphics[width=0.85\textwidth]{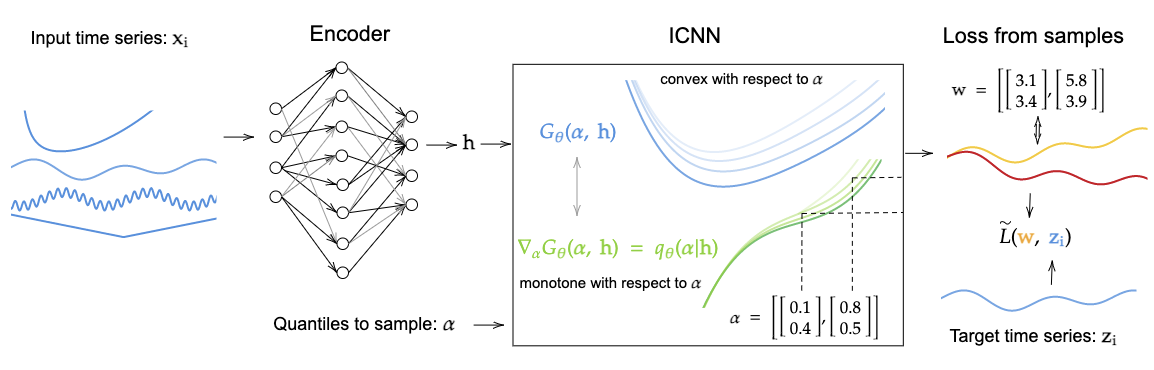}
    \caption{
    Schematic representation of \MQF. The crux of the method is the full multivariate quantile function which is monotone with respect to the multivariate quantile vector $\bfalpha$. It is achieved by modeling the quantile function with the gradient of a function $G_\bftheta$ which is convex with respect to $\bfalpha$.
    The quantile function is conditioned on the past of the time series through a representation obtained from an encoder network.
    The network is trained by minimizing the loss between forecast samples and the true target.
      }
    \label{fig:mqff_fig}
\end{figure*}

\section{MULTIVARIATE QUANTILE FUNCTION FORECASTER}
\label{sec:MQF2}
In this section, we introduce the Multivariate Quantile Function Forecaster (\MQF) which uses a multivariate quantile function conditioned on the past time points to make probabilistic forecasts.
Contemporary deep learning based probabilistic forecasting methods like DeepAR \citep{salinas2020}, MQRNN \citep{wen2017}, or TFT \citep{lim2021} consist of two components: an encoder that extracts features from past observations and compresses them into a finite-dimensional hidden state (\citet{salinas2019} used an RNN-based encoder, \citet{wen2017} used RNNs and CNNs, and \citet{lim2021,eisenach2020mqtransformer} used a Transformer-based architecture), and probabilistic output model which transforms the hidden state into a representation of the probability distribution over future observations (e.g.\ parametric density-based for \citet{salinas2019}, univariate quantile predictions for \citet{wen2017,eisenach2020mqtransformer,lim2021}, normalizing flow-based for \citet{rasul2020}). We follow the same paradigm here and condition the multivariate quantile function (which constitutes the output model) on the hidden state produced by an encoder network. As our focus lies on the output model, we restrict our attention to that component (see e.g.\ \citet{salinas2019,wen2017,benidis2020} for details on the general setup) and only consider the combination with a DeepAR-based encoder (and follow the same window-based training procedure detailed in \citet{salinas2020}), while in principle our approach can be combined with any encoder architecture.
In the remainder of this section, we first present how the PICNN can be used to model a multivariate quantile function. Then, we propose two alternatives to train the model, using the energy score and maximum likelihood, respectively.

\subsection{PICNN Quantile Function}
We propose to use the gradient of a PICNN $g_\bftheta(\bfalpha, \bfh):=\nabla_\bfalpha G_\bftheta(\bfalpha, \bfh)$ to model a conditional multivariate quantile function $q_\bftheta(\bfalpha | \bfh)$ with a quantile vector $\bfalpha \in (0,1)^n$.
We illustrate this in the right half of Figure~\ref{fig:mqff_fig}. The PICNN $G_\bftheta(\bfalpha,\bfh)$ is convex with respect to only $\bfalpha$. Through this setup our multivariate quantile function satisfies the monotonicity property and hence \eqref{eq:property_monotonicity} by design. In addition, the fact that the network is not convex with respect to the second input vector $\bfh$ allows us to flexibly condition the multivariate quantile function on input features or a representation of them produced by a time series encoder model.

As presented in Section \ref{subsection:Multivariate Quantile Functions}, there are two essential properties for a quantile function, the representation property \eqref{eq:multivariate_quantile} and the monotonicity property. Our parametrization through the gradient of the PICNN constrains the multivariate quantile function to fulfill the monotonicity property, which means that we can train our model with standard gradient descent optimizers so as to come as close as possible to the representation property.
To do so, we propose two alternatives: training with the energy score or with maximum likelihood.

While this representation of a multivariate quantile function is general and can be used in any regression context, we propose to use it in the probabilistic forecasting context. We use a forecasting encoder network $H_\bfphi (\bfx) = \bfh$ to obtain a representation of the past time series on which we condition the quantile function. 
To the best of our knowledge, this is the first application of the definition of monotonicity to construct a multivariate quantile function for all quantile levels $\bfalpha \in (0,1)^n$ and the use of ICNN for this application.

\subsection{Training Procedure}

We propose two alternative procedures to train the multivariate quantile forecasting functions described above. First using the energy score to bring the distribution of samples from the model to as close to the true distribution as possible and so fulfill the representation property. The second option is to use normalizing flows as the inverse of the quantile function to map the observed samples to a Gaussian distribution. The network is then trained to maximize the likelihood of the mapped samples under the Gaussian distribution.

\subsubsection{Training via Energy Score}

We propose to train \MQF\ using the energy score~\citep{gneiting2007}, the generalization of CRPS to multivariate distributions.

\paragraph{Training} Consider $m$ training example pairs $\{(\bfx_i, \bfz_i)\}_{i=1}^m$, where $\bfx_i$ denotes the input features and  $\bfz_i$ the target output. Each $\bfz_i$ can span multiple time steps and/or across multiple time series. We minimize the approximated energy score $\tilde{L}_{\text{ES}}$ in
\eqref{eq:approx_energy_score} as
\begin{align*}
&\min_{\bftheta,\bfphi} \frac{1}{m} \sum_{i=1}^{m}\tilde{L}_{\text{ES}}(q_\bftheta(\cdot | H_\bfphi (\bfx_i)),\bfz_i),
\end{align*}
where $q_\bftheta(\cdot | H_\bfphi (\bfx_i))=g_\bftheta( \cdot, H_\bfphi (\bfx_i))$ is the multivariate quantile function. 

\paragraph{Inference} Our multivariate quantile function is trained to provide estimate on all quantile levels, for inference we can compute $\tilde{\bfz} = q_\bftheta(\bfalpha|\bfh)$, where $\bfalpha \in (0,1)^n$ is drawn from a uniform distribution.\footnote{In practice, we use a generalized quantile vector~\citep{carlier2016}, which follows the isotropic Gaussian distribution. Because this empirically allows for a better training.}

\subsubsection{Training via Maximum Likelihood}
The second option is to train the gradient of PICNN through (conditional) normalizing flows. This approach follows \citet{huang2020}, which proposed to use ICNN as normalizing flows. 

We note that normalizing flows take target samples $\bfz$ as input and return the corresponding Gaussian samples, as opposed to quantile functions which take uniform samples and output $\bfz$. To distinguish this reversed direction of mapping, we denote the PICNN used for normalizing flows as $\Flow_{\bftheta}$.

\paragraph{Invertible Gradient} Since the gradient of PICNN is used as normalizing flows, it needs to be invertible. To this end, an $l_2$ term is added to $\bfv_k$, the final layer of the PICNN \eqref{eq:final_PICNN}, i.e.,
\vspace{-2pt}
\begin{equation}
    \Flow_{\bftheta}(\bfz,\bfh) = \bfv_k(\bfz, \bfh) + \frac{\gamma}{2} \| \bfz \|_2^2,
    \label{eq:sc_PICNN}
    \vspace{-1pt}
\end{equation}
where $\gamma>0$ is a trainable parameter. The additional term renders $\Flow_{\bftheta}$ strongly convex, and hence its gradient is invertible. In \citet{huang2020}, they term the mapping of the gradient $\flow_{\bftheta} := \nabla_\bfz \Flow_{\bftheta}$ as convex potential flows.

\paragraph{Training} Given the training example pairs $\{(\bfx_i, \bfz_i)\}_{i=1}^m$, we minimize the negative log-likelihood, with $\tilde{L}_{\text{ML}}$ defined in~\eqref{eq:likelihood}, as
\begin{equation*}
    \min_{\bftheta,\bfphi} \frac{1}{m} \sum_{i=1}^m \tilde{L}_{\text{ML}}(\flow_\bftheta (\bfz_i, H_\bfphi(\bfx_i))).
\vspace{-7pt}
\end{equation*}

\paragraph{Inference} The normalizing flow can serve as a (generalized) quantile function~\citep{carlier2016}, which takes inputs drawn from an isotropic Gaussian distribution and output the predicted target samples.

In particular, we first sample $\bfy \in \mathbb{R}^n$ from the isotropic Gaussian distribution. Then we go backward through the flow to obtain the prediction $\tilde{\bfz}$.
To this end, we solve the convex minimization problem 
\begin{equation}
\min_{\bfz}   \Flow_{\bftheta}(\bfz, \bfh) - \bfz^\top \bfy,
    \label{eq:invert_flow}
\end{equation}
whose minimum $\tilde{\bfz}$ satisfies $\flow_{\bftheta}(\tilde{\bfz}, \bfh) = \bfy$. That is $\tilde{\bfz} = \flow_{\bftheta}^{-1}(\bfy, \bfh)$. The minimization problem~\eqref{eq:invert_flow} is solved using the L-BFGS algorithm~\citep{liu1989}.

\paragraph{Monotonicity} Since we are using the inverse of the normalizing flow as the quantile function, it is important to show that the inverse $\flow_{\bftheta}^{-1}$ is also monotone, i.e., it is the gradient of a convex function. The following proposition guarantees the monotonicity of $\flow_{\bftheta}^{-1}(\cdot, \bfh)$.

\begin{proposition}\label{prop:monotonicity}
Let $\mathcal{D} \subseteq \mathbb{R}^n$ be open, $G: \mathcal{D} \to \mathbb{R}$ be a strongly convex and smooth function and $g$ be its gradient. Then $g^{-1}$ exists and is the gradient of a convex function.
\end{proposition}
For the proof of Proposition~\ref{prop:monotonicity}, we refer the readers to the Appendix. Note that the assumption of smoothness is satisfied when the PICNN architecture uses smooth activation functions such as the softplus function to render the whole network smooth.

\begin{table*}[!ht]
\renewrobustcmd{\bfseries}{\fontseries{b}\selectfont}
\sisetup{detect-weight,mode=text,group-minimum-digits = 2
}
\sisetup{table-figures-uncertainty=1} 

\centering
\scalebox{0.635}{
\begin{tabular}{
@{}ll 
r r r r | r r r r r @{}
 }
 \toprule
Dataset & Model  & \multicolumn{4}{c|}{Metrics over full horizon}  &
 \multicolumn{5}{c}{Mean Quantile Loss over differing forecast horizon} \\
 \cmidrule{3-5}
\cmidrule{6-11} 
& &
  {sum CRPS} &
  {Energy score} &
  {MSIS} &
  {mean\_wQL} &
  {1 step} &
  {5 steps} &
  {10 steps} &
  {15 steps} &
  {20 steps}  \\[0.15em]
  \midrule
&
MQCNN &
2323.5 $\pm$ 54.2 &
1282.1 $\pm$ 1.0 &
11.7 $\pm$ 0.0 &
0.086 $\pm$ 0.0 &
0.042 $\pm$ 0.0 &
0.125 $\pm$ 0.01 &
0.117 $\pm$ 0.02 &
0.094 $\pm$ 0.01 &
\ubold {0.062 $\pm$ 0.0}
\\[0.15em] 
&
DeepAR &
3059.6 $\pm$ 180.8 &
971.7 $\pm$ 59.0 &
7.3 $\pm$ 0.1 &
0.07 $\pm$ 0.0 &
\ubold {0.027 $\pm$ 0.0} &
\ubold {0.055 $\pm$ 0.0} &
\ubold {0.059 $\pm$ 0.0} &
0.074 $\pm$ 0.0 &
0.089 $\pm$ 0.01
\\[0.15em] 
&
\MQF + ES &
\ubold {1723.7 $\pm$ 122.8} &
\ubold {891.1 $\pm$ 32.1} &
\ubold {6.9 $\pm$ 0.1} &
\ubold {0.066 $\pm$ 0.0} &
0.031 $\pm$ 0.0 &
0.08 $\pm$ 0.01 &
0.102 $\pm$ 0.0 &
0.056 $\pm$ 0.0 &
0.068 $\pm$ 0.01
\\[0.15em] 
\multirow{-4}{*}{ \texttt{Elec}} &
\MQF + ML &
2332.523 ± 146.88 &
893.6 ± 53.8 &
7.2 ± 0.6 &
\ubold {0.066 ± 0.0} &
0.038 ± 0.01 &
0.088 ± 0.02 &
0.073 ± 0.01 &
\ubold {0.053 ± 0.0} &
0.067 ± 0.01
\\[0.15em] \hline  
&
MQCNN &
0.419 $\pm$ 0.33 &
0.161 $\pm$ 0.06 &
46.1 $\pm$ 1.6 &
0.993 $\pm$ 0.29 &
0.905 $\pm$ 0.4 &
5.909 $\pm$ 0.37 &
0.878 $\pm$ 0.42 &
0.871 $\pm$ 0.23 &
0.644 $\pm$ 0.18
\\[0.15em] 
&
DeepAR &
0.108 $\pm$ 0.01 &
0.061 $\pm$ 0.0 &
7.2 $\pm$ 0.1 &
0.131 $\pm$ 0.0 &
\ubold {0.074 $\pm$ 0.0} &
\ubold {0.163 $\pm$ 0.0} &
\ubold {0.117 $\pm$ 0.0} &
0.123 $\pm$ 0.0 &
0.144 $\pm$ 0.0
\\[0.15em] 
&
\MQF + ES &
\ubold {0.095 $\pm$ 0.0} &
\ubold {0.06 $\pm$ 0.0} &
7.2 $\pm$ 0.0 &
0.142 $\pm$ 0.0 &
0.104 $\pm$ 0.0 &
0.298 $\pm$ 0.01 &
0.139 $\pm$ 0.01 &
0.127 $\pm$ 0.0 &
\ubold {0.139 $\pm$ 0.01}
\\[0.15em] 
\multirow{-4}{*}{ \texttt{Traf}} &
\MQF + ML &
0.097 ± 0.0 &
0.062 ± 0.0 &
\ubold {6.6 ± 0.1} &
\ubold {0.13 ± 0.0} &
0.078 ± 0.0 &
0.165 ± 0.01 &
0.13 ± 0.0 &
\ubold {0.12 ± 0.0} &
0.14 ± 0.0
\\[0.15em] \hline  
&
MQCNN &
3089.8 $\pm$ 10.4 &
923.1 $\pm$ 3.6 &
41.9 $\pm$ 0.9 &
0.027 $\pm$ 0.0 &
\ubold {0.009 $\pm$ 0.0} &
0.019 $\pm$ 0.0 &
\ubold {0.024 $\pm$ 0.0} &
{-} &
{-}  
\\[0.15em] 
&
DeepAR &
3186.2 $\pm$ 966.5 &
989.3 $\pm$ 244.3 &
50.5 $\pm$ 8.0 &
0.039 $\pm$ 0.01 &
0.015 $\pm$ 0.0 &
0.028 $\pm$ 0.01 &
0.049 $\pm$ 0.01 &
{-} &
{-}  
\\[0.15em] 
&
\MQF + ES &
\ubold {1752.0 $\pm$ 47.3} &
\ubold {619.2 $\pm$ 8.7} &
31.1 $\pm$ 0.3 &
\ubold {0.024 $\pm$ 0.0} &
0.013 $\pm$ 0.0 &
\ubold {0.019 $\pm$ 0.0} &
0.027 $\pm$ 0.0 &
{-} &
{-}  
\\[0.15em] 
\multirow{-4}{*}{ \texttt{M4-daily}} &
\MQF + ML &
1786.028 ± 60.94 &
622.0 ± 14.7 &
\ubold {30.5 ± 0.3} &
\ubold {0.024 ± 0.0} &
0.01 ± 0.0 &
0.019 ± 0.0 &
0.029 ± 0.0 &
{-} &
{-}  
\\[0.15em] \hline  
&
MQCNN &
9196.6 $\pm$ 175.4 &
3269.9 $\pm$ 34.6 &
18.7 $\pm$ 0.4 &
0.12 $\pm$ 0.0 &
0.072 $\pm$ 0.0 &
0.096 $\pm$ 0.0 &
0.115 $\pm$ 0.0 &
\ubold {0.134 $\pm$ 0.0} &
{-}  
\\[0.15em] 
&
DeepAR &
\ubold {7337.0 $\pm$ 345.5} &
2572.1 $\pm$ 95.0 &
14.0 $\pm$ 1.5 &
0.113 $\pm$ 0.0 &
0.063 $\pm$ 0.0 &
0.092 $\pm$ 0.0 &
0.115 $\pm$ 0.0 &
0.143 $\pm$ 0.01 &
{-}  
\\[0.15em] 
&
\MQF + ES &
7365.7 $\pm$ 218.1 &
\ubold {2554.6 $\pm$ 79.3} &
\ubold {12.8 $\pm$ 1.4} &
\ubold {0.112 $\pm$ 0.0} &
\ubold {0.059 $\pm$ 0.0} &
\ubold {0.087 $\pm$ 0.0} &
\ubold {0.113 $\pm$ 0.0} &
0.145 $\pm$ 0.0 &
{-}  
\\[0.15em] 
\multirow{-4}{*}{ \texttt{M4-monthly}} &
\MQF + ML &
8235.445 ± 0.0 &
2839.7 ± 0.0 &
14.4 ± 0.0 &
0.124 ± 0.0 &
0.066 ± 0.0 &
0.1 ± 0.0 &
0.124 ± 0.0 &
0.159 ± 0.0 &
{-}  
\\ [0.15em] \hline
&
MQCNN &
3753.7 $\pm$ 28.1 &
1976.2 $\pm$ 12.8 &
\ubold {34.2 $\pm$ 0.3} &
\ubold {0.115 $\pm$ 0.0} &
\ubold {0.064 $\pm$ 0.0} &
0.141 $\pm$ 0.0 &
{-} &
{-} &
{-}  
\\[0.15em] 
&
DeepAR &
3749.1 $\pm$ 42.6 &
1917.1 $\pm$ 7.9 &
34.9 $\pm$ 0.6 &
0.118 $\pm$ 0.0 &
0.065 $\pm$ 0.0 &
0.145 $\pm$ 0.0 &
{-} &
{-} &
{-}  
\\[0.15em] 
&
\MQF + ES &
\ubold {3649.3 $\pm$ 60.9} &
\ubold {1859.4 $\pm$ 28.1} &
36.7 $\pm$ 1.6 &
0.116 $\pm$ 0.0 &
0.075 $\pm$ 0.0 &
\ubold {0.135 $\pm$ 0.0} &
{-} &
{-} &
{-}  
\\[0.15em] 
\multirow{-4}{*}{ \texttt{M4-yearly}} &
\MQF + ML &
3784.486 ± 105.76 &
1913.2 ± 35.2 &
38.8 ± 2.1 &
0.119 ± 0.0 &
0.07 ± 0.0 &
0.143 ± 0.0 &
{-} &
{-} &
{-}  
\\[0.15em] \bottomrule
\end{tabular}
}
\caption{ Results of \MQF\ compared with other state of the art methods (for all columns lower is better.) We show the mean and standard deviation over 3 training runs. A ``-" indicates that the corresponding time step is beyond the prediction length of the dataset.}
\label{tab:mqf_resuts}
\vspace{-5pt}
\end{table*}

\section{EXPERIMENTS}
\label{sec:experiments}
Our \MQF\ can be multivariate in prediction horizon (multi-horizon) and/or across multiple time series. In our experimental evaluations, we focus on the former case, where at a given time point the model outputs a distribution of multiple points into the future. This evaluation setup allows us to compare to standard univariate forecasting models, here MQCNN \citep{wen2017} and DeepAR \citep{salinas2020}. These two models represent different approaches for multi-horizon predictions. On the one hand, MQCNN factorizes the multivariate distribution over the time steps, considering them independently of each other and therefore the time dependency structure among them is ignored. On the other hand, DeepAR only predicts a single time step at a time and the model is unrolled to predict the full forecast horizon. By doing so it implicitly models the forward dependency among time steps, but at the cost of error accumulating.

We recall that \MQF\ is generic because it can be used in many sequence-to-sequence architectures as an alternative to the decoder. For our experiments, we choose to implement \MQF\ 
on top of a DeepAR encoder. We use the default hyperparameters for the comparison methods as found in GluonTS~\citep{alexandrov2020gluonts}. For \MQF\ we use the default parameters for the DeepAR encoder and PICNN with 40 hidden units and 5 hidden layers for the real experiments, and with 10 hidden units and 2 hidden layers for the synthetic experiments. We train the model to convergence (For real data experiments, we use 100 epochs for MQCNN and DeepAR and 300 for \MQF\ as it is more complex. For synthetic experiments, we use 50 epochs for all models). Otherwise all the hyperparameters are kept constant across models.  \MQF\ is implemented in  PyTorch\footnote{available at \url{https://github.com/awslabs/gluon-ts/tree/master/src/gluonts/torch/model/mqf2}.} \citep{paszke2017automatic}. We refer the readers to the Appendix for more details on the experimental details, model hyperparameters and their robustness.

We evaluate our model on both real and synthetic data. For the real experiments, we evaluate the methods on several real-world datasets and report the performance in terms of various univariate and multivariate metrics. For the synthetic experiment, we test the ability of different models to learn and predict artificial data which follow a Gaussian process.

\subsection{Experiments on Real Data}
We perform experiments on \texttt{Elec} and \texttt{Traf} from the UCI data repository~\citep{Dua:2017}, and different \texttt{M4} competition datasets~\citep{makridakisM4concl}. The results are shown in Table~\ref{tab:mqf_resuts}. Experimental results in terms of more metrics and hyperparameter robustness tests are available in the Appendix. In the following we analyze them along different angles.

\textbf{\MQF\ is competitive with the state of the art.}
Table~\ref{tab:mqf_resuts} shows the mean scaled interval score (MSIS)~\citep{gneiting2007} 
and mean weighted quantile loss, averaged over the $\{0.1, 0.2, ..., 0.9\}$ quantiles and over the full forecast horizon. These are univariate probabilistic forecasting metrics which are computed at each point in the forecast horizon and are averaged over the points.
We observe that \MQF\ is very competitive with the state of the art. In particular, under these two metrics, \MQF\ performs the best in all but 1 dataset (in which \MQF's performance is close to the best one). While a main advantage of our method is to model the time dependencies across the time dimension of the forecast horizon, its performance on modeling the marginal distributions is comparable to that of MQCNN, which by design only learns such distributions.

\begin{figure}[ht]
    \centering
    \begin{subfigure}[b]{0.49\columnwidth}
        \centering
        \includegraphics[width=0.99\columnwidth]{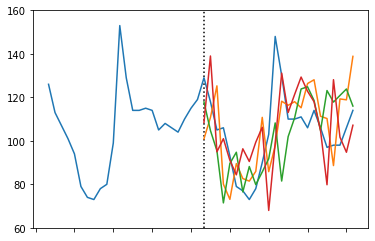}
        \caption{MQCNN}
        \label{fig:sample_path_MQCNN}
    \end{subfigure}
    \begin{subfigure}[b]{0.49\columnwidth}
        \centering
        \includegraphics[width=0.99\columnwidth]{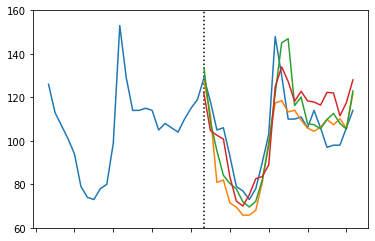}
        \caption{MQF$^2$}
        \label{fig:sample_path_MQF2}
    \end{subfigure}
\caption{Three sample paths generated by MQCNN and MQF$^2$. The dotted vertical lines represent the start of the prediction horizon.}
\vspace{-10pt}
\label{fig:sample_path}
\end{figure}

\begin{figure*}[ht]
    \centering 
    \begin{subfigure}[b]{0.4\columnwidth}
        \centering
        \includegraphics[width=0.99\columnwidth]{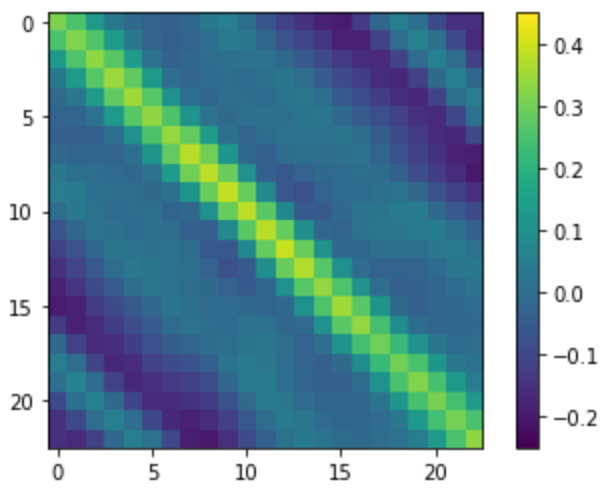}
        \caption{Correlation matrix of the training samples.}
        \label{fig:GP_GT}
    \end{subfigure}
    \begin{subfigure}[b]{0.4\columnwidth}
        \centering
        \includegraphics[width=0.99\columnwidth]{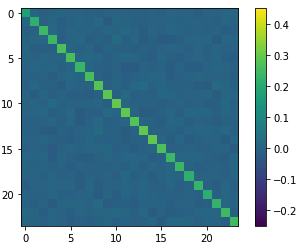}
        \caption{MQCNN \\ MAE: 0.075}
        \label{fig:GP_MQCNN}
    \end{subfigure}
    \begin{subfigure}[b]{0.4\columnwidth}
        \centering
        \includegraphics[width=0.99\columnwidth]{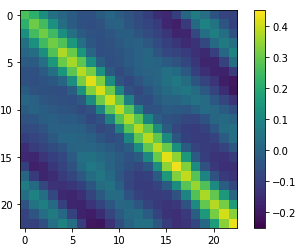}
        \caption{DeepAR \\ MAE: 0.033}
        \label{fig:GP_DEEPAR}
    \end{subfigure}
    \begin{subfigure}[b]{0.4\columnwidth}
        \centering
        \includegraphics[width=0.99\columnwidth]{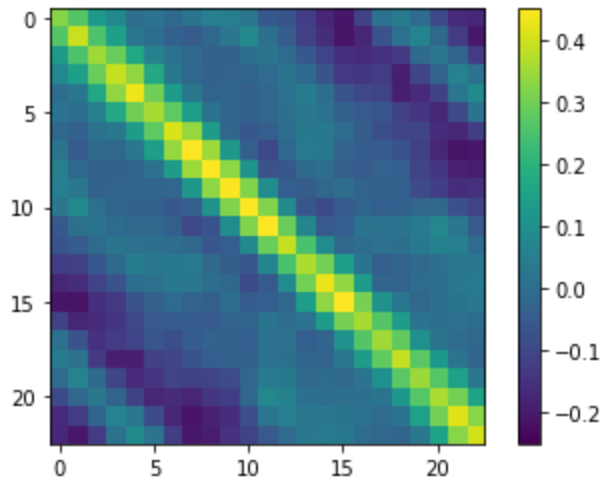}
        \caption{MQF$^2$ + ES \\ MAE: 0.023}
        \label{fig:GP_MQF2_ES}
    \end{subfigure}
    \begin{subfigure}[b]{0.4\columnwidth}
        \centering
        \includegraphics[width=0.99\columnwidth]{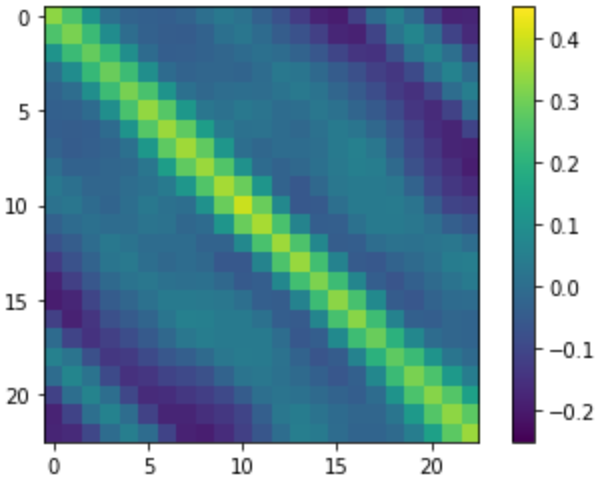}
        \caption{MQF$^2$ + ML \\ MAE: 0.019}
        \label{fig:GP_MQF2_ML}
    \end{subfigure}
    \caption{
    Experiments with generated samples from a Gaussian process over 24 time steps which have the correlation matrix visualized in (a). Figures (b)-(e) show the correlation matrices obtained from 200 samples of different models and the mean absolute error (MAE) between the model correlation matrix and the ground truth.
    }
    \label{fig:correlation_in_samples_synthetic}
    \vspace{-7pt}
\end{figure*}

\textbf{\MQF\ captures the time dependency between outputs.}
We use two multivariate metrics to evaluate the multivariate distributions produced by different models. First, we measure the energy score between samples from the forecasting models and the observed target time series. In addition, we compute CRPS between the sum of these samples and the sum of the observed target time series. The distribution of a sum depends on the dependency among its elements. Hence, accurately measuring the dependency between the time points will result in a better estimate of the distribution of their sum. These two metrics are shown in Table~\ref{tab:mqf_resuts} and are computed over the full forecast horizon of each dataset.
We see that \MQF\ outperforms the comparing methods by some margin, especially when it is trained with energy score. In particular, \MQF\ performs the best in all but 1 result, in which it is very close to the best method and reports a much lower standard deviation over training runs.

We observe that MQCNN is underperforming because it assumes that the time points over the prediction horizon are independent and hence cannot capture time dependency. On the other hand, on some datasets like \texttt{Traf}, \texttt{M4-monthly}, and \texttt{M4-yearly}, DeepAR's implicit modeling of the forward time dependencies allows it to obtain results very close to \MQF.

\textbf{\MQF\ avoids error accumulation.}
DeepAR is able to model the forward dependency across time points implicitly through the unrolling on samples, however this can result in error accumulation through the unrolling \citep{rangapuram2018deep}. To compare DeepAR with our model in this respect, we compute the mean weighted quantile losses on different forecast horizons. Table~\ref{tab:mqf_resuts} shows the loss for 1, 5, 10, 15, and 20 steps ahead. Note that on some datasets the selected steps are longer than the prediction length, and the loss cannot be computed beyond the prediction length. We see that \MQF\ has competitive performance across all time steps and datasets. For all the datasets either MQCNN or \MQF\ perform the best on the furthest quantile horizon, even on datasets where DeepAR performs the best at shorter horizons. However, in the results of \texttt{Traf} dataset, we observe that \MQF\ has a more stable performance than MQCNN, which reports very high losses at all the time steps.

\textbf{\MQF\ produces consistent sample paths.}
Beyond the quantitative evaluation of the multivariate distribution, we evaluate it qualitatively by visually inspecting predicted sample paths. 
In a model where the distribution over each of the time steps is modeled independently, sample paths would fail to represent the dependency between time points which can lead to unrealistic sampled forecasts. Figure~\ref{fig:sample_path} shows sample paths from MQCNN and \MQF\ on the same time series. We observe that the distributions of the samples at each time step are similar for both models. However, the sample paths from MQCNN fail to mimic the smoothness of the real time series, as each time point is modeled and sampled independently. On the contrary, note that the samples from \MQF\ indeed display realistic behavior because of its modeling of the time dependencies. We provide additional visualizations in the Appendix.

\subsection{Experiments on Synthetic Data}
In the real experiments, we observe that our \MQF\ best captures the time dependency structure. Here we further illustrate this advantage using a synthetic dataset of 500 time series of 24 points drawn from a Gaussian process (GP) with a correlation matrix shown in Figure~\ref{fig:correlation_in_samples_synthetic}(a). The kernel of the GP governing the covariance between time points is composed of a radial basis function kernel and a periodic kernel, resulting in a complex correlation structure.

We evaluate how well the different methods can model the marginal GP distribution. We train each of the methods on the GP samples and then generate 200 sample paths for each of them and compute the correlation matrix of the generated sample paths.
The correlation matrices are shown in  Figure~\ref{fig:correlation_in_samples_synthetic}(b)-(e).
In addition to the visualization, we compute the mean absolute error (MAE) between the correlation matrix from model samples and the true correlation matrix. If a method captures the GP well, it will generate sample paths which closely follow the distribution and hence report a correlation matrix similar to the true one. 

We see that MQCNN generates a correlation matrix which is essentially diagonal and reports the highest error. This shows that it fails to capture the correlation, as it assumes each time point to be independent and therefore ignores the time dependency structure.
For the DeepAR method, its unrolling mechanism allows it to capture the correlation matrix reasonably well and report a much lower error than MQCNN. Finally, as our \MQF\ explicitly considers the whole sample path at once, it best approximates the true correlation matrix and has the lowest errors.

\section{DISCUSSION}
\label{sec:conclusion}
In this paper, we presented \MQF, a novel method for probabilistic forecasting via a multivariate quantile function that we model as the gradient of an input convex neural network. Our experiments show that we maintain favorable properties of prior work on (univariate) quantile functions for probabilistic forecasts while addressing some of their shortcomings. In particular, sample paths (which are a commonly-used way of passing probabilistic forecasts to downstream components) can easily be generated from our model and correctly reflect the dependency structure across time (which also makes them visually coherent). Further, there is no accumulation of forecast error over the length of the forecast horizon and our method is overall very competitive with the state of the art. 

Despite these benefits, there are situations and applications where alternative approaches might be better suited. In particular, autoregressive constructions that decompose the joint distribution into its telescoping univariate marginals \citep{wei2008, uria2013, papamakarios2017, wang2019, jaini2019}, allow the quantile levels to retain their classical probabilistic interpretation (e.g.\ for the construction of univariate prediction intervals) and provide direct access to certain conditional distributions of interest (future conditioned on past). Similarly, multi-horizon approaches provide direct access to the univariate marginal distributions, which in our approach can only be obtained through sampling. In fact, an interesting avenue for future work is to explore whether a multivariate quantile function model can be constructed that retains the ability to access marginal and conditional distributions without resorting to sampling.  
Future work could further extend our approach to the practically important case of count distributions and assess the quality of our approach for quantile functions jointly over the time and item dimensions. Finally, more suitable forecasters in domain adaptation \citep{jin2022domain} with faster training schemes \citep{pmlr-v139-lu21d} can be developed, ultimately being able to be incorporated for downstream decision makings, e.g., planning cloud computing and vehicle controllers \citep{park2019linear, park2020structured, kim2020optimal}.


\subsubsection*{Acknowledgements}
The authors would like to thank the five anonymous referees for their thorough review and constructive suggestions. They would also like to thank Michael Bohlke-Schneider, Syama Sundar Rangapuram, Lorenzo Stella, and Jasper Zschiegner for reviewing the code and giving useful advice. Moreover, they would like to thank Levon Nurbekyan and Samy Wu Fung for the helpful discussion.

\bibliography{main}

\begin{thebibliography}{}

\bibitem[Alexandrov et~al., 2020]{alexandrov2020gluonts}
Alexandrov, A., Benidis, K., Bohlke-Schneider, M., Flunkert, V., Gasthaus, J.,
  Januschowski, T., Maddix, D.~C., Rangapuram, S., Salinas, D., Schulz, J.,
  Stella, L., T{\"u}rkmen, A.~C., and Wang, Y. (2020).
\newblock Gluon{TS}: Probabilistic and neural time series modeling in python.
\newblock {\em Journal of Machine Learning Research}, 21(116):1--6.

\bibitem[Amos et~al., 2017]{amos2017}
Amos, B., Xu, L., and Kolter, J.~Z. (2017).
\newblock Input convex neural networks.
\newblock In {\em International Conference on Machine Learning}, pages
  146--155. PMLR.

\bibitem[Benidis et~al., 2020]{benidis2020}
Benidis, K., Rangapuram, S.~S., Flunkert, V., Wang, B., Maddix, D.~C.,
  T{\"{u}}rkmen, A.~C., Gasthaus, J., Bohlke{-}Schneider, M., Salinas, D.,
  Stella, L., Callot, L., and Januschowski, T. (2020).
\newblock Neural forecasting: Introduction and literature overview.
\newblock {\em CoRR}, abs/2004.10240.

\bibitem[Brenier, 1991]{brenier1991}
Brenier, Y. (1991).
\newblock Polar factorization and monotone rearrangement of vector-valued
  functions.
\newblock {\em Communications on pure and applied mathematics}, 44(4):375--417.

\bibitem[Bunne et~al., 2021]{bunne2021}
Bunne, C., Meng-Papaxanthos, L., Krause, A., and Cuturi, M. (2021).
\newblock Jkonet: Proximal optimal transport modeling of population dynamics.
\newblock {\em arXiv preprint arXiv:2106.06345}.

\bibitem[Carlier et~al., 2016]{carlier2016}
Carlier, G., Chernozhukov, V., Galichon, A., et~al. (2016).
\newblock Vector quantile regression: an optimal transport approach.
\newblock {\em Annals of Statistics}, 44(3):1165--1192.

\bibitem[Chen et~al., 2019]{chen2018}
Chen, Y., Shi, Y., and Zhang, B. (2019).
\newblock Optimal control via neural networks: A convex approach.
\newblock In {\em International Conference on Learning Representations}.

\bibitem[Chernozhukov et~al., 2017]{chernozhukov2017}
Chernozhukov, V., Galichon, A., Hallin, M., and Henry, M. (2017).
\newblock {Monge–Kantorovich depth, quantiles, ranks and signs}.
\newblock {\em The Annals of Statistics}, 45(1):223 -- 256.

\bibitem[Dabney et~al., 2018]{dabney2018}
Dabney, W., Ostrovski, G., Silver, D., and Munos, R. (2018).
\newblock Implicit quantile networks for distributional reinforcement learning.
\newblock In Dy, J. and Krause, A., editors, {\em Proceedings of the 35th
  International Conference on Machine Learning}, volume~80 of {\em Proceedings
  of Machine Learning Research}, pages 1096--1105. PMLR.

\bibitem[de~B{\'e}zenac et~al., 2020]{de2020normalizing}
de~B{\'e}zenac, E., Rangapuram, S.~S., Benidis, K., Bohlke-Schneider, M.,
  Kurle, R., Stella, L., Hasson, H., Gallinari, P., and Januschowski, T.
  (2020).
\newblock Normalizing {K}alman filters for multivariate time series analysis.
\newblock {\em Advances in Neural Information Processing Systems}, 33.

\bibitem[Dheeru and Karra~Taniskidou, 2017]{Dua:2017}
Dheeru, D. and Karra~Taniskidou, E. (2017).
\newblock {UCI} machine learning repository.

\bibitem[Eisenach et~al., 2020]{eisenach2020mqtransformer}
Eisenach, C., Patel, Y., and Madeka, D. (2020).
\newblock Mqtransformer: Multi-horizon forecasts with context dependent and
  feedback-aware attention.
\newblock {\em arXiv preprint arXiv:2009.14799}.

\bibitem[Gasthaus et~al., 2019]{gasthaus2019}
Gasthaus, J., Benidis, K., Wang, Y., Rangapuram, S.~S., Salinas, D., Flunkert,
  V., and Januschowski, T. (2019).
\newblock Probabilistic forecasting with spline quantile function rnns.
\newblock In {\em The 22nd international conference on artificial intelligence
  and statistics}, pages 1901--1910. PMLR.

\bibitem[Gneiting and Raftery, 2007]{gneiting2007}
Gneiting, T. and Raftery, A.~E. (2007).
\newblock Strictly proper scoring rules, prediction, and estimation.
\newblock {\em Journal of the American statistical Association},
  102(477):359--378.

\bibitem[Gouttes et~al., 2021]{gouttes2021}
Gouttes, A., Rasul, K., Koren, M., Stephan, J., and Naghibi, T. (2021).
\newblock Probabilistic time series forecasting with implicit quantile
  networks.
\newblock {\em arXiv preprint arXiv:2107.03743}.

\bibitem[Hallin et~al., 2021]{hallin2021}
Hallin, M., del Barrio, E., Cuesta-Albertos, J., and Matrán, C. (2021).
\newblock {Distribution and quantile functions, ranks and signs in dimension d:
  A measure transportation approach}.
\newblock {\em The Annals of Statistics}, 49(2):1139 -- 1165.

\bibitem[Hanasusanto et~al., 2015]{hanasusantoKWZ15}
Hanasusanto, G.~A., Kuhn, D., Wallace, S.~W., and Zymler, S. (2015).
\newblock Distributionally robust multi-item newsvendor problems with
  multimodal demand distributions.
\newblock {\em Math. Program.}, 152(1-2):1--32.

\bibitem[Huang et~al., 2021]{huang2020}
Huang, C.-W., Chen, R. T.~Q., Tsirigotis, C., and Courville, A. (2021).
\newblock Convex potential flows: Universal probability distributions with
  optimal transport and convex optimization.
\newblock In {\em International Conference on Learning Representations}.

\bibitem[Jaini et~al., 2019]{jaini2019}
Jaini, P., Selby, K.~A., and Yu, Y. (2019).
\newblock Sum-of-squares polynomial flow.
\newblock In Chaudhuri, K. and Salakhutdinov, R., editors, {\em Proceedings of
  the 36th International Conference on Machine Learning}, volume~97 of {\em
  Proceedings of Machine Learning Research}, pages 3009--3018. PMLR.

\bibitem[Jin et~al., 2022]{jin2022domain}
Jin, X., Park, Y., Maddix, D.~C., Wang, H., and Wang, Y. (2022).
\newblock Domain adaptation for time series forecasting via attention sharing.

\bibitem[Kass and Vos, 2011]{kass2011}
Kass, R.~E. and Vos, P.~W. (2011).
\newblock {\em Geometrical foundations of asymptotic inference}, volume 908.
\newblock John Wiley \& Sons.

\bibitem[Kim et~al., 2020]{kim2020optimal}
Kim, J., Park, Y., Fox, J.~D., Boyd, S.~P., and Dally, W. (2020).
\newblock Optimal operation of a plug-in hybrid vehicle with battery thermal
  and degradation model.
\newblock In {\em 2020 American Control Conference (ACC)}, pages 3083--3090.
  IEEE.

\bibitem[Knott and Smith, 1984]{knott1984}
Knott, M. and Smith, C.~S. (1984).
\newblock On the optimal mapping of distributions.
\newblock {\em Journal of Optimization Theory and Applications}, 43(1):39--49.

\bibitem[Kobyzev et~al., 2021]{kobyzev2021}
Kobyzev, I., Prince, S.~J., and Brubaker, M.~A. (2021).
\newblock Normalizing flows: An introduction and review of current methods.
\newblock {\em IEEE Transactions on Pattern Analysis and Machine Intelligence},
  43(11):3964–3979.

\bibitem[Koenker, 2005]{koenker2005}
Koenker, R. (2005).
\newblock {\em Quantile Regression}.
\newblock Econometric Society Monographs. Cambridge University Press.

\bibitem[Koenker and Bassett, 1978]{koenker1978}
Koenker, R. and Bassett, G. (1978).
\newblock Regression quantiles.
\newblock {\em Econometrica}, 46(1):33--50.

\bibitem[Koltchinskii, 1997]{koltchinskii1997}
Koltchinskii, V.~I. (1997).
\newblock {M-estimation, convexity and quantiles}.
\newblock {\em The Annals of Statistics}, 25(2):435 -- 477.

\bibitem[Li et~al., 2019]{li2019}
Li, S., Jin, X., Xuan, Y., Zhou, X., Chen, W., Wang, Y.-X., and Yan, X. (2019).
\newblock Enhancing the locality and breaking the memory bottleneck of
  transformer on time series forecasting.
\newblock In Wallach, H., Larochelle, H., Beygelzimer, A., d\textquotesingle
  Alch\'{e}-Buc, F., Fox, E., and Garnett, R., editors, {\em Advances in Neural
  Information Processing Systems}, volume~32. Curran Associates, Inc.

\bibitem[Lim et~al., 2021]{lim2021}
Lim, B., Arif, S., Loeff, N., and Pfister, T. (2021).
\newblock Temporal fusion transformers for interpretable multi-horizon time
  series forecasting.
\newblock {\em International Journal of Forecasting}, 37(4):1748--1764.

\bibitem[Liu and Nocedal, 1989]{liu1989}
Liu, D.~C. and Nocedal, J. (1989).
\newblock On the limited memory {BFGS} method for large scale optimization.
\newblock {\em Mathematical programming}, 45(1):503--528.

\bibitem[Lu et~al., 2021]{pmlr-v139-lu21d}
Lu, Y., Park, Y., Chen, L., Wang, Y., De~Sa, C., and Foster, D. (2021).
\newblock Variance reduced training with stratified sampling for forecasting
  models.
\newblock In Meila, M. and Zhang, T., editors, {\em Proceedings of the 38th
  International Conference on Machine Learning}, volume 139 of {\em Proceedings
  of Machine Learning Research}, pages 7145--7155. PMLR.

\bibitem[Makkuva et~al., 2020]{makkuva2020}
Makkuva, A., Taghvaei, A., Oh, S., and Lee, J. (2020).
\newblock Optimal transport mapping via input convex neural networks.
\newblock In {\em International Conference on Machine Learning}, pages
  6672--6681. PMLR.

\bibitem[Makridakis et~al., 2018]{makridakisM4concl}
Makridakis, S. et~al. (2018).
\newblock The {M4} competition: Results, findings, conclusion and way forward.
\newblock {\em International Journal of Forecasting}, 34(4):802--808.

\bibitem[Matheson and Winkler, 1976]{matheson76}
Matheson, J.~E. and Winkler, R.~L. (1976).
\newblock Scoring rules for continuous probability distributions.
\newblock {\em Management Science}, 22(10):1087--1096.

\bibitem[Onken et~al., 2021]{onken2021ot}
Onken, D., Wu~Fung, S., Li, X., and Ruthotto, L. (2021).
\newblock Ot-flow: Fast and accurate continuous normalizing flows via optimal
  transport.
\newblock In {\em Proceedings of the AAAI Conference on Artificial
  Intelligence}, volume~35, pages 9223--9232.

\bibitem[Papamakarios et~al., 2017]{papamakarios2017}
Papamakarios, G., Pavlakou, T., and Murray, I. (2017).
\newblock Masked autoregressive flow for density estimation.
\newblock In Guyon, I., Luxburg, U.~V., Bengio, S., Wallach, H., Fergus, R.,
  Vishwanathan, S., and Garnett, R., editors, {\em Advances in Neural
  Information Processing Systems}, volume~30. Curran Associates, Inc.

\bibitem[Park et~al., 2021]{park2021}
Park, Y., Maddix, D., Aubet, F.-X., Kan, K., Gasthaus, J., and Wang, Y. (2021).
\newblock Learning quantile functions without quantile crossing for
  distribution-free time series forecasting.
\newblock {\em arXiv preprint arXiv:2111.06581}.

\bibitem[Park et~al., 2019]{park2019linear}
Park, Y., Mahadik, K., Rossi, R.~A., Wu, G., and Zhao, H. (2019).
\newblock Linear quadratic regulator for resource-efficient cloud services.
\newblock In {\em Proceedings of the ACM Symposium on Cloud Computing}, pages
  488--489.

\bibitem[Park et~al., 2020]{park2020structured}
Park, Y., Rossi, R., Wen, Z., Wu, G., and Zhao, H. (2020).
\newblock Structured policy iteration for linear quadratic regulator.
\newblock In {\em International Conference on Machine Learning}, pages
  7521--7531. PMLR.

\bibitem[Paszke et~al., 2019]{paszke2017automatic}
Paszke, A., Gross, S., Massa, F., Lerer, A., Bradbury, J., Chanan, G., Killeen,
  T., Lin, Z., Gimelshein, N., Antiga, L., Desmaison, A., Kopf, A., Yang, E.,
  DeVito, Z., Raison, M., Tejani, A., Chilamkurthy, S., Steiner, B., Fang, L.,
  Bai, J., and Chintala, S. (2019).
\newblock Pytorch: An imperative style, high-performance deep learning library.
\newblock {\em Advances in Neural Information Processing Systems},
  32:8024--8035.

\bibitem[Petropoulos et~al., 2021]{petropoulos2021forecasting}
Petropoulos, F., Apiletti, D., Assimakopoulos, V., Babai, M.~Z., Barrow, D.~K.,
  Taieb, S.~B., Bergmeir, C., Bessa, R.~J., Bijak, J., Boylan, J.~E., Browell,
  J., Carnevale, C., Castle, J.~L., Cirillo, P., Clements, M.~P., Cordeiro, C.,
  Oliveira, F. L.~C., Baets, S.~D., Dokumentov, A., Ellison, J., Fiszeder, P.,
  Franses, P.~H., Frazier, D.~T., Gilliland, M., Gönül, M.~S., Goodwin, P.,
  Grossi, L., Grushka-Cockayne, Y., Guidolin, M., Guidolin, M., Gunter, U.,
  Guo, X., Guseo, R., Harvey, N., Hendry, D.~F., Hollyman, R., Januschowski,
  T., Jeon, J., Jose, V. R.~R., Kang, Y., Koehler, A.~B., Kolassa, S.,
  Kourentzes, N., Leva, S., Li, F., Litsiou, K., Makridakis, S., Martin, G.~M.,
  Martinez, A.~B., Meeran, S., Modis, T., Nikolopoulos, K., Önkal, D.,
  Paccagnini, A., Panagiotelis, A., Panapakidis, I., Pavía, J.~M., Pedio, M.,
  Pedregal, D.~J., Pinson, P., Ramos, P., Rapach, D.~E., Reade, J.~J.,
  Rostami-Tabar, B., Rubaszek, M., Sermpinis, G., Shang, H.~L., Spiliotis, E.,
  Syntetos, A.~A., Talagala, P.~D., Talagala, T.~S., Tashman, L., Thomakos, D.,
  Thorarinsdottir, T., Todini, E., Arenas, J. R.~T., Wang, X., Winkler, R.~L.,
  Yusupova, A., and Ziel, F. (2021).
\newblock Forecasting: theory and practice.

\bibitem[Peyr{\'e} et~al., 2019]{peyre2019}
Peyr{\'e}, G., Cuturi, M., et~al. (2019).
\newblock Computational optimal transport: With applications to data science.
\newblock {\em Foundations and Trends in Machine Learning}, 11(5-6):355--607.

\bibitem[Rajaram and Tang, 2001]{RAJARAM2001582}
Rajaram, K. and Tang, C.~S. (2001).
\newblock The impact of product substitution on retail merchandising.
\newblock {\em European Journal of Operational Research}, 135(3):582--601.

\bibitem[Rangapuram et~al., 2018]{rangapuram2018deep}
Rangapuram, S.~S., Seeger, M.~W., Gasthaus, J., Stella, L., Wang, Y., and
  Januschowski, T. (2018).
\newblock Deep state space models for time series forecasting.
\newblock {\em Advances in neural information processing systems},
  31:7785--7794.

\bibitem[Rasul et~al., 2021]{rasul2021autoregressive}
Rasul, K., Seward, C., Schuster, I., and Vollgraf, R. (2021).
\newblock Autoregressive denoising diffusion models for multivariate
  probabilistic time series forecasting.

\bibitem[Rasul et~al., 2020]{rasul2020}
Rasul, K., Sheikh, A.-S., Schuster, I., Bergmann, U., and Vollgraf, R. (2020).
\newblock Multivariate probabilistic time series forecasting via conditioned
  normalizing flows.
\newblock {\em arXiv preprint arXiv:2002.06103}.

\bibitem[Ruthotto and Haber, 2021]{ruthotto2021}
Ruthotto, L. and Haber, E. (2021).
\newblock An introduction to deep generative modeling.
\newblock {\em GAMM-Mitteilungen}, 44(2):e202100008.

\bibitem[Salinas et~al., 2019]{salinas2019}
Salinas, D., Bohlke-Schneider, M., Callot, L., Medico, R., and Gasthaus, J.
  (2019).
\newblock High-dimensional multivariate forecasting with low-rank gaussian
  copula processes.
\newblock {\em Advances in neural information processing systems}, 32.

\bibitem[Salinas et~al., 2020]{salinas2020}
Salinas, D., Flunkert, V., Gasthaus, J., and Januschowski, T. (2020).
\newblock Deepar: Probabilistic forecasting with autoregressive recurrent
  networks.
\newblock {\em International Journal of Forecasting}, 36(3):1181--1191.

\bibitem[Sz{\'e}kely, 2003]{szekely2003}
Sz{\'e}kely, G.~J. (2003).
\newblock E-statistics: The energy of statistical samples.
\newblock {\em Bowling Green State University, Department of Mathematics and
  Statistics Technical Report}, 3(05):1--18.

\bibitem[Tabak and Turner, 2013]{tabak2013}
Tabak, E.~G. and Turner, C.~V. (2013).
\newblock A family of nonparametric density estimation algorithms.
\newblock {\em Communications on Pure and Applied Mathematics}, 66(2):145--164.

\bibitem[Tagasovska and Lopez-Paz, 2019]{tagasovska2019}
Tagasovska, N. and Lopez-Paz, D. (2019).
\newblock Single-model uncertainties for deep learning.
\newblock In Wallach, H., Larochelle, H., Beygelzimer, A., d\textquotesingle
  Alch\'{e}-Buc, F., Fox, E., and Garnett, R., editors, {\em Advances in Neural
  Information Processing Systems}, volume~32. Curran Associates, Inc.

\bibitem[Uria et~al., 2013]{uria2013}
Uria, B., Murray, I., and Larochelle, H. (2013).
\newblock Rnade: The real-valued neural autoregressive density-estimator.
\newblock In Burges, C. J.~C., Bottou, L., Welling, M., Ghahramani, Z., and
  Weinberger, K.~Q., editors, {\em Advances in Neural Information Processing
  Systems}, volume~26. Curran Associates, Inc.

\bibitem[Villani, 2009]{villani2009optimal}
Villani, C. (2009).
\newblock {\em Optimal transport: old and new}.
\newblock Springer.

\bibitem[Wang et~al., 2019]{wang2019}
Wang, J., Sun, S., and Yu, Y. (2019).
\newblock Multivariate triangular quantile maps for novelty detection.
\newblock In Wallach, H., Larochelle, H., Beygelzimer, A., d\textquotesingle
  Alch\'{e}-Buc, F., Fox, E., and Garnett, R., editors, {\em Advances in Neural
  Information Processing Systems}, volume~32. Curran Associates, Inc.

\bibitem[Wei, 2008]{wei2008}
Wei, Y. (2008).
\newblock An approach to multivariate covariate-dependent quantile contours
  with application to bivariate conditional growth charts.
\newblock {\em Journal of the American Statistical Association},
  103(481):397--409.

\bibitem[Wen et~al., 2017]{wen2017}
Wen, R., Torkkola, K., Narayanaswamy, B., and Madeka, D. (2017).
\newblock A multi-horizon quantile recurrent forecaster.
\newblock In {\em NIPS 2017 Time Series Workshop}.

\bibitem[Zhang et~al., 2014]{zhang2014}
Zhang, R.-Q., Zhang, L.-K., Zhou, W.-H., Saigal, R., and Wang, H.-W. (2014).
\newblock The multi-item newsvendor model with cross-selling and the solution
  when demand is jointly normally distributed.
\newblock {\em European Journal of Operational Research}, 236(1):147--159.

\end{thebibliography}

\clearpage
\appendix

\thispagestyle{empty}

\onecolumn \makesupplementtitle

\section{EXPERIMENTAL DETAILS}
\subsection{Datasets}
Table~\ref{tab:datasets} lists the information of the datasets used in the experiments. The datasets are available in the GluonTS dataset repository.\footnote{\url{https://github.com/awslabs/gluon-ts/blob/master/src/gluonts/dataset/repository/datasets.py.}}
\begin{table}[h]
\centering
\begin{tabular}{|l|l|c|c|c|c|c|c|c|c}
\hline
\textsc{domain} & \textsc{name} & \textsc{support} & \textsc{freq} & \textsc{no. ts} & \textsc{avg. len} &  \textsc{pred. len} & \textsc{no. covariates} 

\\
\hline
\hline
 electrical load & \texttt{Elec} & $\mathbb{R}^+$ & H & 321 & 21044 & 24 & 4\\ \hline 
road traffic & \texttt{Traf} & $[0, 1]$ & H & 862 & 14036 & 24 & 4\\\hline 
\multirow{4}{*}{\begin{tabular}[c]{@{}l@{}}M4 forecasting \\ competition \\   \end{tabular}} & \texttt{M4-daily} & $\mathbb{R}^+$ & D & 4227  & 2357 &14 & 3\\
 & \texttt{M4-weekly} & $\mathbb{R}^+$ &  W & 359 & 1022 & 13 & 2\\
  & \texttt{M4-monthly} & $\mathbb{R}^+$ & M & 48000 & 216 & 18&1\\
  & \texttt{M4-quarterly} & $\mathbb{R}^+$ & Q  & 24000 & 92 & 8& 1\\
    & \texttt{M4-yearly} & $\mathbb{R}^+$ & Y & 23000 & 31 & 6 & 0\\
\hline
\end{tabular}
\caption{Summary of dataset statistics, where \texttt{Elec} and \texttt{Traf} are dervied from the UCI data repository~\citep{Dua:2017}, and \texttt{M4} are competition datasets~\citep{makridakisM4concl}.}
\label{tab:datasets}
\end{table}

\subsection{Hyperparameters}
The hyperparameters used in the experiments are listed in Table~\ref{table:hyperparameters}. For the RNN parameters we use the default setting of the \texttt{DeepAREstimator} in the GluonTS package~\citep{alexandrov2020gluonts}. The other hyperparameters were selected by performing a grid search only on the \texttt{Elec}
dataset, and were used as default values on all the other datasets.
\begin{table}[h]
\centering
\begin{tabular}{|l|l|c|}
\hline
\textsc{Type} & \textsc{Hyperparameter} & \textsc{Value}
\\
\hline
\hline
 \multirow{2}{*}{\begin{tabular}[c]{@{}l@{}} RNN   \end{tabular}}& layers & 2\\ 
 & nodes & 40 \\ \hline
\multirow{2}{*}{\begin{tabular}[c]{@{}l@{}} PICNN   \end{tabular}}& layers & 5\\ 
 & nodes & 40 \\ \hline
 Energy Score & num. of samples & 50\\ \hline
 \multirow{2}{*}{\begin{tabular}[c]{@{}l@{}} Training \end{tabular}} & epochs & 100 / 300\\
 & batch size & 32 \\
\hline
\end{tabular}
\caption{Summary of hyperparameters. For the number of training epochs, 100 is used for DeepAR and MQCNN, and 300 is used for MQF$^2$. This is because DeepAR and MQCNN have already converged after 100 epochs and MQF$^2$ takes more epochs to converge.}
\label{table:hyperparameters}
\end{table}

\section{DEFINITION OF EVALUATION METRICS}

Consider the target value $z_{i,t}$ for the $i$-th time series at time $t$, where $i=1,...,m$ and $t=T+1,...,T+\tau$, and the corresponding predictions $\{\hat z_{i,t,s}\}_{s=1}^{S}$ from $S$ sample paths. We denote the $\alpha$-quantile of the predictions as $\hat z_{i,t}^\alpha$.

\subsection{Mean Weighted Quantile Loss}

The $\alpha$-quantile loss is defined as
\begin{equation*}
\rho_{\alpha}(z_{i,t}, \hat z_{i,t}^\alpha) = (z_{i,t} - \hat z_{i,t}^\alpha)(\alpha - \mathbf{1}\{z_{i,t} - \hat z_{i,t}^\alpha < 0 \}).
\end{equation*}
The mean weighted quantile loss is defined as
\begin{equation*}
\text{mean}\_\text{wQL} =  \frac{1}{|\mathcal{A}|}\sum_{\alpha \in \mathcal{A}} 
\frac{\sum_{i=1}^m \sum_{t=T}^{T+\tau} 2\rho_\alpha(z_{i,t}, \hat z_{i,t}^\alpha)} {\sum_{i=1}^m \sum_{t=T}^{T+\tau} |z_{i,t}|},
\end{equation*}
where $\mathcal{A}$ is a set of prespecified quantile levels. In our experiments, we used $\mathcal{A}=\{0.1, 0.2, 0.3, 0.4, 0.5, 0.6, 0.7, 0.8, 0.9\}$.


\subsection{$n$-th Step Mean Weighted Quantile Loss}
The mean weighted quantile loss at the $n$-th step is defined as 
\begin{equation*}
    \text{mean}\_\text{wQL}(n) = \frac{1}{|\mathcal{A}|}\sum_{\alpha \in \mathcal{A}} 
\frac{\sum_{i=1}^m 2\rho_\alpha(z_{i,n}, \hat z_{i,n}^\alpha)} {\sum_{i=1}^m |z_{i,n}|}.
\end{equation*}

\subsection{Sum CRPS}
The sum CRPS is the (approximated) CRPS for the sum of the predictions over the prediction horizon and defined as
\begin{equation*}
    \text{sum}\_\text{CRPS} = \frac{1}{m} \sum_{i=1}^m \left( -\frac{1}{2|S|^2} \sum_{j=1}^S \sum_{k=1}^S |\hat u_{i,j} - \hat u_{i,k}| + \frac{1}{|S|} \sum_{j=1}^S |\hat u_{i,j} - u_i| \right),
\end{equation*}
where $\hat u_{i,j} = \sum_{t} \hat z_{i,t,j}$ and $u_i = \sum_{t} u_{i,t}$.

\subsection{Mean Scaled Interval Score}
The mean scaled interval score (MSIS) is defined as
\[ 
\begin{aligned}
\text{MSIS}(\zeta) =
  \frac{1}{\text{SE}(z)}\Big( \frac{1}{m\tau} \sum_{i=1}^{m} \sum_{t=T+1}^{T+\tau} (\hat z_{i,t}^{\alpha_U} - \hat z_{i,t}^{\alpha_L} + &\frac{2}{\zeta}[(\hat z_{i,t}^{\alpha_L}-z_{i,t})\mathbf{1}\{z_{i,t} < \hat z_{i,t}^{\alpha_L}\} + (z_{i,t}-\hat z_{i,t}^{\alpha_U})\mathbf{1}\{z_{i,t}>\hat z_{i,t}^{\alpha_U}\}]\Big),
\end{aligned}
\]
where the upper quantile $\alpha_U = 1-\zeta/2$, and the lower quantile $\alpha_L = \zeta/2$. The seasonal error SE for time series frequency $f$ is given by
\[
\text{SE}(z) = \frac{1}{m(T-f)}\sum_{i=1}^m \sum_{t'=1}^{T-f} |z_{i,t'} - z_{i, t'+f}|.
\]

\section{PROOF OF PROPOSITION 1}
Here, we state Proposition 1 again and provide the proof.
\setcounter{proposition}{0}
\begin{proposition}
Let $\mathcal{D} \subseteq \mathbb{R}^n$ be open, $G: \mathcal{D} \to \mathbb{R}$ be a strongly convex and smooth function and $g$ be its gradient. Then $g^{-1}$ exists and is the gradient of a convex function.
\end{proposition}

\begin{proof}
The strong convexity and smoothness of $G$ implies the existence of $g^{-1}$ and that $\nabla g(\bfx)$ is symmetric positive definite (SPD) for all $\bfx \in \mathcal{D}$. Since $g$ is one-to-one, smooth, and $\nabla g(\bfx)$ is SPD for all $\bfx \in \mathcal{D}$, by \citet[Corollary A.2]{kass2011}, $g^{-1}$ is also smooth and therefore $\nabla g^{-1} (g(\bfx)) \nabla g(\bfx) = \bfI_n$ for all $\bfx \in \mathcal{D}$. This implies $\nabla g^{-1} (\bfy)$ is SPD for all $\bfy \in g(\mathcal{D})$ and hence
\begin{equation}\label{eq:symmetry_2nd_deri}
    \frac{\partial [g^{-1}]_i}{\partial y_j} =\frac{\partial [g^{-1}]_j}{\partial y_i} \quad \text{for all } i,j=1,2,...,n.
\end{equation}
Let $\psi(\bfy) = \sum_{j=1}^{n} y_j \int_0^1 [g^{-1}]_j (t\bfy) dt$. Consider its partial derivative
\begin{align*}
    \frac{\partial \psi}{\partial y_i} (\bfy) &= \int_0^1 [g^{-1}]_i (t\bfy) dt + \int_0^1 \sum_{j=1}^{n} y_j t \frac{\partial [g^{-1}]_j}{\partial y_i} (t\bfy) dt \\
    \intertext{using~\eqref{eq:symmetry_2nd_deri}, we get}
     &= \int_0^1 [g^{-1}]_i (t\bfy) dt + \int_0^1 \sum_{j=1}^{n} y_j t \frac{\partial [g^{-1}]_i}{\partial y_j} (t\bfy) dt \\
     \intertext{applying the chain rule $\frac{\partial}{\partial t} \Big( [g^{-1}]_i (t\bfy) \Big) = t \frac{\partial [g^{-1}]_i}{\partial y_j} (t\bfy)$, we obtain}
    &= \int_0^1 [g^{-1}]_i (t\bfy) dt + \int_0^1 t \frac{\partial}{\partial t} \Big( [g^{-1}]_i (t\bfy) \Big) dt \\
    \intertext{performing integration by parts, we have}
    &= \int_0^1 [g^{-1}]_i (t\bfy) dt + \left. t [g^{-1}]_i (t\bfy) \right|_{t=0}^{t=1} - \int_0^1 [g^{-1}]_i (t\bfy) dt \\
    &= [g^{-1}]_i (\bfy), \quad \text{for} \quad i=1,2,...,n.
\end{align*}

Therefore, $\nabla \psi=g^{-1}$. Moreover $\psi$ is convex because its Hessian $\nabla g^{-1} (\bfy)$ is SPD for all $\bfy \in g(\mathcal{D})$. Therefore, $g^{-1}$ is the gradient of the convex function $\psi$.
\end{proof}

\section{ADDITIONAL RESULTS TABLE}

\begin{table*}[!ht]
\renewrobustcmd{\bfseries}{\fontseries{b}\selectfont}
\sisetup{detect-weight,mode=text,group-minimum-digits = 2
}
\sisetup{table-figures-uncertainty=1} 

\centering
\scalebox{0.635}{
\begin{tabular}{
@{}ll 
r r r r | r r r r r @{}
 }
 \toprule
Dataset & Model  & \multicolumn{4}{c|}{Metrics over full horizon}  &
 \multicolumn{5}{c}{Mean Quantile Loss over differing forecast horizon} \\
 \cmidrule{3-5}
\cmidrule{6-11} 
& &
  {sum CRPS} &
  {Energy score} &
  {MSIS} &
  {mean\_wQL} &
  {1 step} &
  {5 steps} &
  {10 steps} &
  {15 steps} &
  {20 steps}  \\ 
  \midrule
&
MQCNN &
2323.5 $\pm$ 54.2 &
1282.1 $\pm$ 1.0 &
11.7 $\pm$ 0.0 &
0.086 $\pm$ 0.0 &
0.042 $\pm$ 0.0 &
0.125 $\pm$ 0.01 &
0.117 $\pm$ 0.02 &
0.094 $\pm$ 0.01 &
\ubold {0.062 $\pm$ 0.0}
\\  
&
DeepAR &
3059.6 $\pm$ 180.8 &
971.7 $\pm$ 59.0 &
7.3 $\pm$ 0.1 &
0.07 $\pm$ 0.0 &
\ubold {0.027 $\pm$ 0.0} &
\ubold {0.055 $\pm$ 0.0} &
\ubold {0.059 $\pm$ 0.0} &
0.074 $\pm$ 0.0 &
0.089 $\pm$ 0.01
\\  
&
\MQF + ES &
\ubold {1723.7 $\pm$ 122.8} &
\ubold {891.1 $\pm$ 32.1} &
\ubold {6.9 $\pm$ 0.1} &
\ubold {0.066 $\pm$ 0.0} &
0.031 $\pm$ 0.0 &
0.08 $\pm$ 0.01 &
0.102 $\pm$ 0.0 &
0.056 $\pm$ 0.0 &
0.068 $\pm$ 0.01
\\  
\multirow{-4}{*}{ \texttt{Elec}} &
\MQF + ML &
2332.523 ± 146.88 &
893.6 ± 53.8 &
7.2 ± 0.6 &
\ubold {0.066 ± 0.0} &
0.038 ± 0.01 &
0.088 ± 0.02 &
0.073 ± 0.01 &
\ubold {0.053 ± 0.0} &
0.067 ± 0.01
\\  \hline  
&
MQCNN &
0.419 $\pm$ 0.33 &
0.161 $\pm$ 0.06 &
46.1 $\pm$ 1.6 &
0.993 $\pm$ 0.29 &
0.905 $\pm$ 0.4 &
5.909 $\pm$ 0.37 &
0.878 $\pm$ 0.42 &
0.871 $\pm$ 0.23 &
0.644 $\pm$ 0.18
\\  
&
DeepAR &
0.108 $\pm$ 0.01 &
0.061 $\pm$ 0.0 &
7.2 $\pm$ 0.1 &
0.131 $\pm$ 0.0 &
\ubold {0.074 $\pm$ 0.0} &
\ubold {0.163 $\pm$ 0.0} &
\ubold {0.117 $\pm$ 0.0} &
0.123 $\pm$ 0.0 &
0.144 $\pm$ 0.0
\\  
&
\MQF + ES &
\ubold {0.095 $\pm$ 0.0} &
\ubold {0.06 $\pm$ 0.0} &
7.2 $\pm$ 0.0 &
0.142 $\pm$ 0.0 &
0.104 $\pm$ 0.0 &
0.298 $\pm$ 0.01 &
0.139 $\pm$ 0.01 &
0.127 $\pm$ 0.0 &
\ubold {0.139 $\pm$ 0.01}
\\  
\multirow{-4}{*}{ \texttt{Traf}} &
\MQF + ML &
0.097 ± 0.0 &
0.062 ± 0.0 &
\ubold {6.6 ± 0.1} &
\ubold {0.13 ± 0.0} &
0.078 ± 0.0 &
0.165 ± 0.01 &
0.13 ± 0.0 &
\ubold {0.12 ± 0.0} &
0.14 ± 0.0
\\  \hline  
&
MQCNN &
3089.8 $\pm$ 10.4 &
923.1 $\pm$ 3.6 &
41.9 $\pm$ 0.9 &
0.027 $\pm$ 0.0 &
\ubold {0.009 $\pm$ 0.0} &
0.019 $\pm$ 0.0 &
\ubold {0.024 $\pm$ 0.0} &
{-} &
{-}  
\\  
&
DeepAR &
3186.2 $\pm$ 966.5 &
989.3 $\pm$ 244.3 &
50.5 $\pm$ 8.0 &
0.039 $\pm$ 0.01 &
0.015 $\pm$ 0.0 &
0.028 $\pm$ 0.01 &
0.049 $\pm$ 0.01 &
{-} &
{-}  
\\  
&
\MQF + ES &
\ubold {1752.0 $\pm$ 47.3} &
\ubold {619.2 $\pm$ 8.7} &
31.1 $\pm$ 0.3 &
\ubold {0.024 $\pm$ 0.0} &
0.013 $\pm$ 0.0 &
\ubold {0.019 $\pm$ 0.0} &
0.027 $\pm$ 0.0 &
{-} &
{-}  
\\  
\multirow{-4}{*}{ \texttt{M4-daily}} &
\MQF + ML &
1786.028 ± 60.94 &
622.0 ± 14.7 &
\ubold {30.5 ± 0.3} &
\ubold {0.024 ± 0.0} &
0.01 ± 0.0 &
0.019 ± 0.0 &
0.029 ± 0.0 &
{-} &
{-}  
\\  \hline  
&
MQCNN &
3572.45 ± 104.73 &
1463.4 ± 18.8 &
62.3 ± 3.2 &
0.065 ± 0.0 &
0.045 ± 0.0 &
0.067 ± 0.0 &
0.072 ± 0.0 &
{-} &
{-}
\\ 
&
DeepAR &
2885.141 ± 443.69 &
1166.9 ± 119.6 &
26.0 ± 4.5 &
0.054 ± 0.01 &
\ubold {0.034 ± 0.0} &
\ubold {0.053 ± 0.0} &
0.062 ± 0.01 &
{-} &
{-}
\\ 
&
MQF$^2$ + ES &
2831.64 ± 175.06 &
1122.6 ± 30.6 &
\ubold {21.5 ± 1.5} &
\ubold {0.052 ± 0.0} &
0.043 ± 0.0 &
\ubold {0.053 ± 0.0} &
0.056 ± 0.0 &
{-} &
{-}
\\ 
\multirow{-4}{*}{ \texttt{M4-weekly}} &
MQF$^2$ + ML &
\ubold {2577.461 ± 107.32} &
\ubold {1107.9 ± 32.5} &
26.1 ± 1.8 &
\ubold {0.052 ± 0.0} &
0.039 ± 0.0 &
0.054 ± 0.0 &
\ubold {0.054 ± 0.0} &
{-} &
{-}
\\ \hline 
&
MQCNN &
9196.6 $\pm$ 175.4 &
3269.9 $\pm$ 34.6 &
18.7 $\pm$ 0.4 &
0.12 $\pm$ 0.0 &
0.072 $\pm$ 0.0 &
0.096 $\pm$ 0.0 &
0.115 $\pm$ 0.0 &
\ubold {0.134 $\pm$ 0.0} &
{-}  
\\  
&
DeepAR &
\ubold {7337.0 $\pm$ 345.5} &
2572.1 $\pm$ 95.0 &
14.0 $\pm$ 1.5 &
0.113 $\pm$ 0.0 &
0.063 $\pm$ 0.0 &
0.092 $\pm$ 0.0 &
0.115 $\pm$ 0.0 &
0.143 $\pm$ 0.01 &
{-}  
\\  
&
\MQF + ES &
7365.7 $\pm$ 218.1 &
\ubold {2554.6 $\pm$ 79.3} &
\ubold {12.8 $\pm$ 1.4} &
\ubold {0.112 $\pm$ 0.0} &
\ubold {0.059 $\pm$ 0.0} &
\ubold {0.087 $\pm$ 0.0} &
\ubold {0.113 $\pm$ 0.0} &
0.145 $\pm$ 0.0 &
{-}  
\\  
\multirow{-4}{*}{ \texttt{M4-monthly}} &
\MQF + ML &
8235.445 ± 0.0 &
2839.7 ± 0.0 &
14.4 ± 0.0 &
0.124 ± 0.0 &
0.066 ± 0.0 &
0.1 ± 0.0 &
0.124 ± 0.0 &
0.159 ± 0.0 &
{-}  
\\ \hline
&
MQCNN &
3348.365 ± 53.53 &
1718.6 ± 26.4 &
15.6 ± 1.6 &
0.09 ± 0.0 &
0.059 ± 0.0 &
0.096 ± 0.0 &
{-} &
{-} &
{-}
\\ 
&
DeepAR &
3184.673 ± 46.2 &
1575.3 ± 34.2 &
15.0 ± 2.7 &
\ubold {0.085 ± 0.0} &
\ubold {0.051 ± 0.0} &
\ubold {0.092 ± 0.0} &
{-} &
{-} &
{-}
\\ 
&
MQF$^2$ + ES &
\ubold {3134.404 ± 235.3} &
\ubold {1533.2 ± 81.5} &
11.8 ± 0.5 &
\ubold {0.085 ± 0.0} &
0.054 ± 0.0 &
\ubold {0.092 ± 0.01} &
{-} &
{-} &
{-}
\\ 
\multirow{-4}{*}{ \texttt{M4-quarterly}} &
MQF$^2$ + ML &
3338.119 ± 135.57 &
1591.5 ± 47.0 &
\ubold {11.7 ± 0.8} &
0.088 ± 0.0 &
0.053 ± 0.0 &
0.095 ± 0.0 &
{-} &
{-} &
{-}
\\ \hline 
&
MQCNN &
3753.7 $\pm$ 28.1 &
1976.2 $\pm$ 12.8 &
\ubold {34.2 $\pm$ 0.3} &
\ubold {0.115 $\pm$ 0.0} &
\ubold {0.064 $\pm$ 0.0} &
0.141 $\pm$ 0.0 &
{-} &
{-} &
{-}  
\\  
&
DeepAR &
3749.1 $\pm$ 42.6 &
1917.1 $\pm$ 7.9 &
34.9 $\pm$ 0.6 &
0.118 $\pm$ 0.0 &
0.065 $\pm$ 0.0 &
0.145 $\pm$ 0.0 &
{-} &
{-} &
{-}  
\\  
&
\MQF + ES &
\ubold {3649.3 $\pm$ 60.9} &
\ubold {1859.4 $\pm$ 28.1} &
36.7 $\pm$ 1.6 &
0.116 $\pm$ 0.0 &
0.075 $\pm$ 0.0 &
\ubold {0.135 $\pm$ 0.0} &
{-} &
{-} &
{-}  
\\  
\multirow{-4}{*}{ \texttt{M4-yearly}} &
\MQF + ML &
3784.486 ± 105.76 &
1913.2 ± 35.2 &
38.8 ± 2.1 &
0.119 ± 0.0 &
0.07 ± 0.0 &
0.143 ± 0.0 &
{-} &
{-} &
{-}  
\\  \bottomrule
\end{tabular}
}
\caption{Results (with additional datasets) of \MQF\ compared with other state of the art methods (for all columns lower is better.) We show the mean and standard deviation over 3 training runs. A ``-" indicates that the corresponding time step is beyond the prediction length of the dataset.}
\label{tab:mqf_resuts_add_data}
\end{table*}

\begin{table*}[!ht]
\renewrobustcmd{\bfseries}{\fontseries{b}\selectfont}
\sisetup{detect-weight,mode=text,group-minimum-digits = 2
}
\sisetup{table-figures-uncertainty=1} 

\centering
\scalebox{0.7}{
\begin{tabular}{
@{}ll 
r r r | r r r | r r r @{}
 }
 \toprule
Dataset & Model  & \multicolumn{3}{c|}{Point forecast metrics}  &
 \multicolumn{4}{c}{Probabilistic metrics} \\
 \cmidrule{3-5}
\cmidrule{6-9} 
& &
  {MASE} &
  {sMAPE} &
  {NRMSE} &
  {wQL 0.1} &
  {wQL 0.5} &
  {wQL 0.9} &
  {MAE coverage}  \\ 
  \midrule
&
MQCNN &
1.179 ± 0.0 &
0.168 ± 0.0 &
0.843 ± 0.0 &
0.055 ± 0.0 &
0.107 ± 0.0 &
0.055 ± 0.0 &
\ubold {0.034 ± 0.0}
\\ 
&
DeepAR &
0.95 ± 0.0 &
0.129 ± 0.0 &
0.757 ± 0.0 &
\ubold {0.034 ± 0.0} &
\ubold {0.082 ± 0.0} &
0.049 ± 0.0 &
0.205 ± 0.0
\\ 
&
MQF$^2$ + ES &
1.568 ± 1.143 &
0.177 ± 0.093 &
1.157 ± 0.953 &
0.077 ± 0.073 &
0.153 ± 0.127 &
0.072 ± 0.051 &
0.074 ± 0.011
\\ 
\multirow{-4}{*}{ \texttt{Elec}} &
MQF$^2$ + ML &
\ubold {0.918 ± 0.051} &
\ubold {0.121 ± 0.004} &
\ubold {0.647 ± 0.037} &
0.036 ± 0.002 &
0.083 ± 0.004 &
\ubold {0.045 ± 0.004} &
0.105 ± 0.008
\\ \hline  
&
MQCNN &
3.155 ± 0.0 &
0.998 ± 0.0 &
0.892 ± 0.0 &
0.864 ± 0.0 &
0.655 ± 0.0 &
2.197 ± 0.0 &
0.454 ± 0.0
\\ 
&
DeepAR &
\ubold {0.598 ± 0.0} &
0.157 ± 0.0 &
0.417 ± 0.0 &
0.071 ± 0.0 &
\ubold {0.156 ± 0.0} &
\ubold {0.107 ± 0.0} &
0.046 ± 0.0
\\ 
&
MQF$^2$ + ES &
0.667 ± 0.014 &
0.2 ± 0.002 &
\ubold {0.407 ± 0.004} &
0.074 ± 0.002 &
0.171 ± 0.003 &
0.118 ± 0.002 &
\ubold {0.038 ± 0.013}
\\ 
\multirow{-4}{*}{ \texttt{Traf}} &
MQF$^2$ + ML &
0.604 ± 0.004 &
\ubold {0.156 ± 0.001} &
0.415 ± 0.003 &
\ubold {0.064 ± 0.001} &
0.156 ± 0.001 &
0.111 ± 0.002 &
0.046 ± 0.025
\\ \hline  
&
MQCNN &
3.892 ± 0.0 &
0.035 ± 0.0 &
0.108 ± 0.0 &
0.021 ± 0.0 &
0.032 ± 0.0 &
0.016 ± 0.0 &
\ubold {0.03 ± 0.0}
\\ 
&
DeepAR &
4.256 ± 0.0 &
0.038 ± 0.0 &
0.107 ± 0.0 &
0.022 ± 0.0 &
0.034 ± 0.0 &
0.017 ± 0.0 &
0.036 ± 0.0
\\ 
&
MQF$^2$ + ES &
3.76 ± 0.125 &
0.035 ± 0.001 &
0.103 ± 0.001 &
0.018 ± 0.0 &
0.031 ± 0.001 &
0.014 ± 0.001 &
0.079 ± 0.022
\\ 
\multirow{-4}{*}{ \texttt{M4-daily}} &
MQF$^2$ + ML &
\ubold {3.584 ± 0.134} &
\ubold {0.034 ± 0.001} &
\ubold {0.102 ± 0.002} &
\ubold {0.017 ± 0.0} &
\ubold {0.029 ± 0.001} &
\ubold {0.013 ± 0.0} &
0.077 ± 0.047
\\ \hline  
&
MQCNN &
3.463 ± 0.0 &
0.098 ± 0.0 &
0.133 ± 0.0 &
0.039 ± 0.0 &
0.069 ± 0.0 &
0.057 ± 0.0 &
0.082 ± 0.0
\\ 
&
DeepAR &
3.362 ± 0.0 &
0.093 ± 0.0 &
0.127 ± 0.0 &
0.029 ± 0.0 &
0.069 ± 0.0 &
0.037 ± 0.0 &
0.133 ± 0.0
\\ 
&
MQF$^2$ + ES &
\ubold {3.005 ± 0.16} &
0.09 ± 0.005 &
\ubold {0.12 ± 0.002} &
\ubold {0.026 ± 0.001} &
0.066 ± 0.002 &
0.039 ± 0.004 &
0.072 ± 0.034
\\ 
\multirow{-4}{*}{ \texttt{M4-weekly}} &
MQF$^2$ + ML &
3.135 ± 0.19 &
\ubold {0.087 ± 0.004} &
0.124 ± 0.004 &
0.027 ± 0.001 &
\ubold {0.065 ± 0.002} &
\ubold {0.035 ± 0.001} &
\ubold {0.06 ± 0.019}
\\ \hline  
&
MQCNN &
1.217 ± 0.0 &
0.146 ± 0.0 &
0.306 ± 0.0 &
0.103 ± 0.0 &
0.133 ± 0.0 &
0.092 ± 0.0 &
\ubold {0.091 ± 0.0}
\\ 
&
DeepAR &
1.255 ± 0.0 &
\ubold {0.145 ± 0.0} &
\ubold {0.294 ± 0.0} &
\ubold {0.069 ± 0.0} &
\ubold {0.13 ± 0.0} &
0.084 ± 0.0 &
0.1 ± 0.0
\\ 
&
MQF$^2$ + ES &
\ubold {1.141 ± 0.022} &
0.149 ± 0.002 &
0.306 ± 0.006 &
0.073 ± 0.004 &
0.133 ± 0.002 &
\ubold {0.083 ± 0.005} &
0.097 ± 0.018
\\ 
\multirow{-4}{*}{ \texttt{M4-monthly}} &
MQF$^2$ + ML &
1.29 ± 0.0 &
0.165 ± 0.0 &
0.327 ± 0.0 &
0.08 ± 0.0 &
0.146 ± 0.0 &
0.098 ± 0.0 &
0.092 ± 0.0
\\ \hline  
&
MQCNN &
1.64 ± 0.0 &
0.122 ± 0.0 &
0.244 ± 0.0 &
0.055 ± 0.0 &
0.116 ± 0.0 &
0.064 ± 0.0 &
0.147 ± 0.0
\\ 
&
DeepAR &
\ubold {1.306 ± 0.0} &
\ubold {0.108 ± 0.0} &
\ubold {0.233 ± 0.0} &
\ubold {0.049 ± 0.0} &
\ubold {0.101 ± 0.0} &
0.06 ± 0.0 &
\ubold {0.026 ± 0.0}
\\ 
&
MQF$^2$ + ES &
1.364 ± 0.092 &
0.112 ± 0.005 &
0.235 ± 0.006 &
0.05 ± 0.003 &
0.104 ± 0.006 &
0.059 ± 0.004 &
0.067 ± 0.029
\\ 
\multirow{-4}{*}{ \texttt{M4-quarterly}} &
MQF$^2$ + ML &
1.444 ± 0.08 &
0.12 ± 0.004 &
0.244 ± 0.005 &
0.053 ± 0.003 &
0.11 ± 0.003 &
\ubold {0.058 ± 0.002} &
0.073 ± 0.035
\\ \hline  
&
MQCNN &
3.358 ± 0.0 &
\ubold {0.14 ± 0.0} &
\ubold {0.286 ± 0.0} &
0.086 ± 0.0 &
\ubold {0.134 ± 0.0} &
\ubold {0.09 ± 0.0} &
0.132 ± 0.0
\\ 
&
DeepAR &
\ubold {3.235 ± 0.0} &
0.14 ± 0.0 &
0.295 ± 0.0 &
\ubold {0.066 ± 0.0} &
0.138 ± 0.0 &
0.103 ± 0.0 &
\ubold {0.059 ± 0.0}
\\ 
&
MQF$^2$ + ES &
3.442 ± 0.137 &
0.146 ± 0.005 &
0.292 ± 0.005 &
0.071 ± 0.008 &
0.141 ± 0.003 &
0.1 ± 0.007 &
0.108 ± 0.024
\\ 
\multirow{-4}{*}{ \texttt{M4-yearly}} &
MQF$^2$ + ML &
3.507 ± 0.124 &
0.15 ± 0.004 &
0.295 ± 0.004 &
0.071 ± 0.004 &
0.143 ± 0.003 &
0.093 ± 0.003 &
0.083 ± 0.028
\\ \bottomrule
\end{tabular}
}
\caption{Results (with additional metrics) of \MQF\ compared with other state of the art methods. We show the mean and standard deviation over 3 training runs. A ``-" indicates that the corresponding time step is beyond the prediction length of the dataset.}
\label{tab:mqf_resuts_add_metric}
\end{table*}

\section{ROBUSTNESS EXPERIMENTS}
To show that our \MQF\ is robust with respect to its hyperparameters, we perform experiments with varying encoder hidden state size and ICNN hidden layer size. We perform 3 training runs, and the mean and standard deviation over the runs are reported in Figures~\ref{fig:ablation_hidden_size}-\ref{fig:ablation_ICNN_hidden_size}. We observe that the performance of our \MQF\ is steady when the hyperparameters are changing. We also see that the standard deviations are small relative to the mean. In particular, the magnitudes of the standard deviations are about 15\% of that of the means. When the encoder hidden state and ICNN hidden layer sizes increase, the performance gets slightly better in general.

\begin{figure}[ht]
    \centering
    \begin{subfigure}[b]{0.48\columnwidth}
        \centering
        \includegraphics[width=0.95\textwidth]{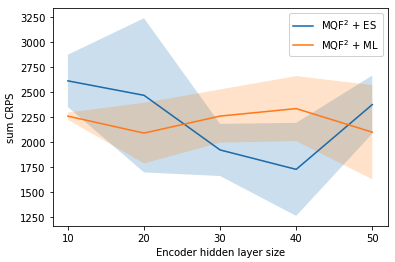}
        \caption{sum CRPS as we vary the hidden state size}
    \end{subfigure}
    \hfill
    \begin{subfigure}[b]{0.48\columnwidth}
        \centering
        \includegraphics[width=0.95\textwidth]{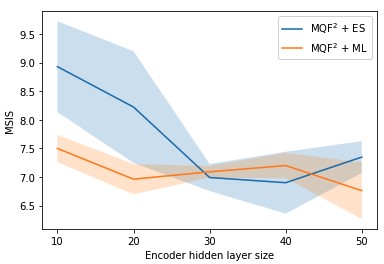}
        \caption{MSIS as we vary the hidden state size}
    \end{subfigure}
    \hfill
    \begin{subfigure}[b]{0.48\columnwidth}
        \centering
        \includegraphics[width=0.95\textwidth]{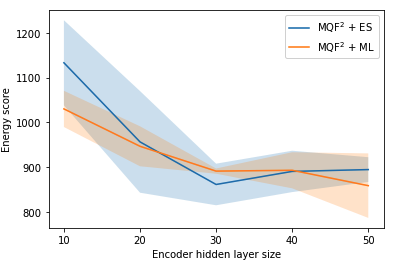}
        \caption{ES as we vary the hidden state size}
    \end{subfigure}
    \hfill
    \begin{subfigure}[b]{0.48\columnwidth}
        \centering
        \includegraphics[width=0.95\textwidth]{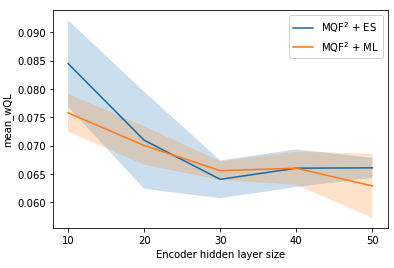}
        \caption{wQL as we vary the hidden state size}
    \end{subfigure}
\caption{To test the robustness of our \MQF, we investigate the influence of the size of the encoder's hidden state on its performance. We report the sum CRPS, MSIS, Energy Score and mean weighted quantile loss when different hidden state sizes are used. The experiments are repeated 3 times. The solid lines represent the mean of the results, and the colored regions represent the range of 1 standard deviation.}
\label{fig:ablation_hidden_size}
\end{figure}

\begin{figure}[ht]
    \centering
    \begin{subfigure}[b]{0.48\columnwidth}
        \centering
        \includegraphics[width=0.95\textwidth]{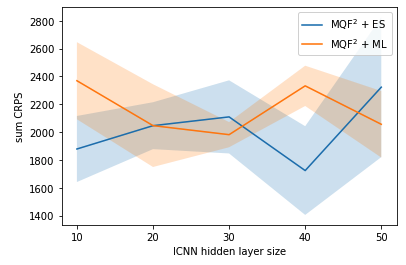}
        \caption{sum CRPS as we vary the hidden layer size}
    \end{subfigure}
    \hfill
    \begin{subfigure}[b]{0.48\columnwidth}
        \centering
        \includegraphics[width=0.95\textwidth]{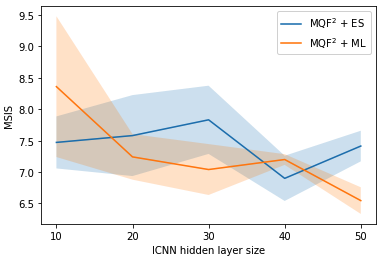}
        \caption{MSIS as we vary the hidden layer size}
    \end{subfigure}
    \hfill
    \begin{subfigure}[b]{0.48\columnwidth}
        \centering
        \includegraphics[width=0.95\textwidth]{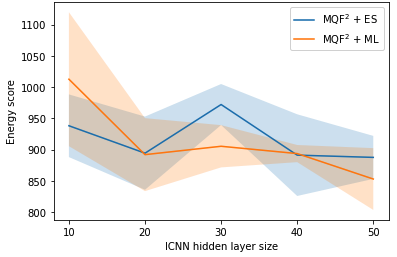}
        \caption{ES as we vary the hidden layer size}
    \end{subfigure}
    \hfill
    \begin{subfigure}[b]{0.48\columnwidth}
        \centering
        \includegraphics[width=0.95\textwidth]{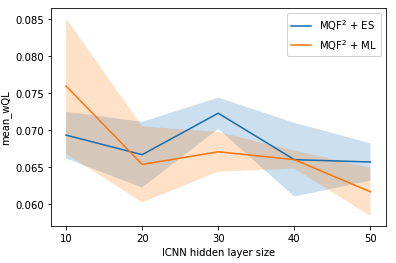}
        \caption{wQL as we vary the hidden layer size}
    \end{subfigure}
\caption{To test the robustness of our \MQF, we investigate the influence of the width of the ICNN on its performance. We report the sum CRPS, MSIS, Energy Score and mean weighted quantile loss when different hidden layer sizes are used. The experiments are repeated 3 times. The solid lines represent the mean of the results, and the colored regions represent the range of 1 standard deviation.}
\label{fig:ablation_ICNN_hidden_size}
\end{figure}

\clearpage
\section{SAMPLE PATH FIGURES}
\begin{figure}[ht]
    \centering
    \begin{subfigure}[b]{0.9\columnwidth}
        \centering
        \includegraphics[width=0.85\textwidth]{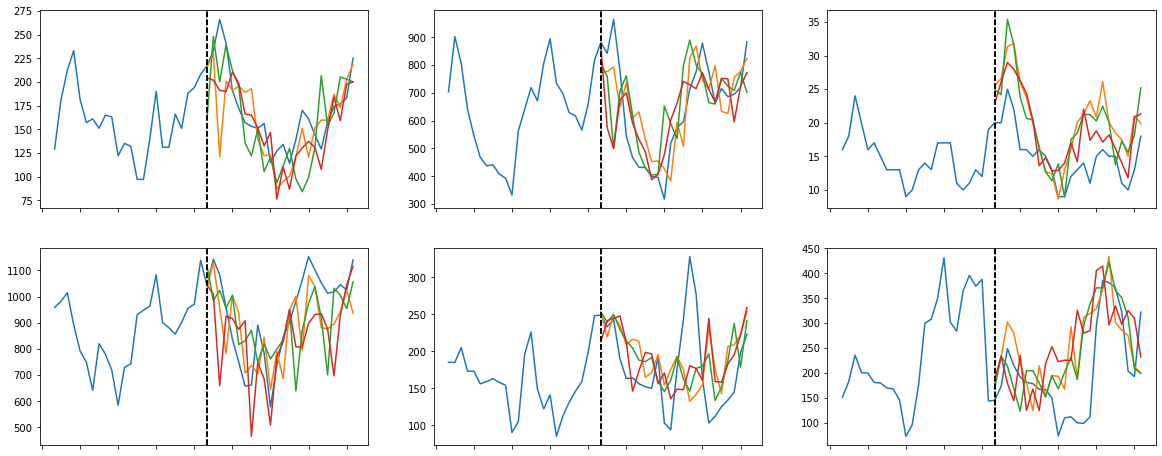}
        \caption{MQCNN}
    \end{subfigure}
    \hfill
    \begin{subfigure}[b]{0.9\columnwidth}
        \centering
        \includegraphics[width=0.85\textwidth]{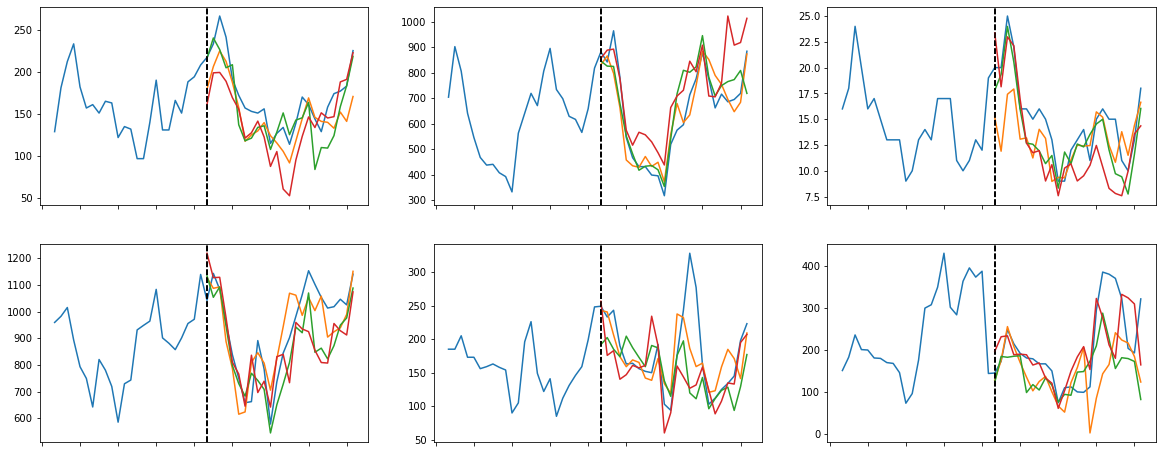}
        \caption{DeepAR}
    \end{subfigure}
    \hfill
    \begin{subfigure}[b]{0.9\columnwidth}
        \centering
        \includegraphics[width=0.85\textwidth]{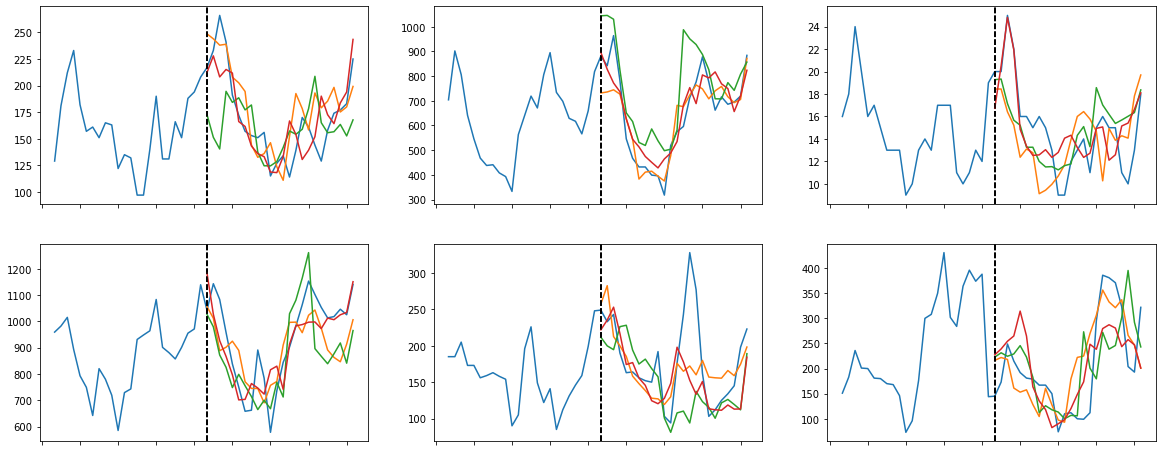}
        \caption{MQF$^2$}
    \end{subfigure}
\caption{Sample paths generated by MQCNN, DeepAR and MQF$^2$. Three sample paths are generated for each of the 6 time series of the \texttt{Elec} dataset. The dotted vertical lines represent the start of the prediction horizon.}
\label{fig:sample_path_extra}
\end{figure}

\vfill

\end{document}